\documentclass{article}

\usepackage[preprint]{preprint}

\usepackage[utf8]{inputenc}
\usepackage[T1]{fontenc}
\usepackage{amsmath,amssymb,amsfonts,amsthm}
\usepackage{algorithm}
\usepackage{algorithmic}
\usepackage{booktabs}
\usepackage{graphicx}
\usepackage[hidelinks]{hyperref}
\usepackage{xcolor}
\usepackage{multirow}
\usepackage{makecell}
\usepackage{subcaption}
\usepackage{float}
\usepackage{tikz}
\usetikzlibrary{positioning, arrows.meta, fit, calc, decorations.pathreplacing}

\newcommand{\grace}{\textsc{GRACE}}
\newcommand{\RR}{\mathbb{R}}
\newcommand{\EE}{\mathbb{E}}
\newcommand{\PP}{\mathbb{P}}
\newcommand{\bx}{\mathbf{x}}

\newcommand{\bW}{\mathbf{W}}

\DeclareMathOperator{\sigmoid}{sigmoid}

\DeclareMathOperator{\ReLU}{ReLU}

\newtheorem{proposition}{Proposition}
\newtheorem{corollary}{Corollary}
\newtheorem{definition}{Definition}
\newtheorem*{remark}{Remark}

\title{\grace{}: Gated Refinement for Accurate Causal Edge Discovery in High-Dimensional Time Series}

\author{
  Mohammad Fesanghary \\
  \texttt{mfesanghary1@bloomberg.net}
  \And
  Abhinav Havaldar \\
  \texttt{ahavaldar@gmail.com}
}

\date{}

\begin{document}
\maketitle

\begin{abstract}
From climate teleconnections to gene regulation, modern time-series datasets encompass tens or hundreds of interacting variables, making causal discovery increasingly challenging.
Constraint-based methods offer statistical rigor but their nonlinear CI tests are infeasible at scale, while score-based alternatives avoid CI testing but require arbitrary thresholds to binarize continuous edge scores.
We propose \grace{} (\textbf{G}ated \textbf{R}efinement for \textbf{A}ccurate \textbf{C}ausal \textbf{E}dge discovery), which refines constraint-based discovery using Hard Concrete gates with $L_0$ regularization---each candidate edge has an independent gate whose values concentrate near 0 or 1, yielding a clean bimodal separation that makes the binary decision robust---unlike the narrow, overlapping score distributions produced by $L_1$ and attention-based methods.
A fast linear CI skeleton provides high-recall candidates; a single gated model then prunes false positives by learning which edges genuinely improve prediction, with automatic regularization adapted to problem dimensions and skeleton density.
Systematic experiments on synthetic benchmarks---spanning diverse graph topologies (scale-free, Erd\H{o}s--R\'enyi, small-world) and dimensionalities up to $d=100$---show that \grace{} substantially improves F1 over its base CI method while maintaining high precision, and outperforms attention-based and score-based alternatives.
\grace{} matches or exceeds expensive nonlinear CI tests at a fraction of the cost ($75\times$ faster).
On a real-world river flow dataset---where rainfall confounders, variable propagation lags, and distributional shifts violate standard assumptions---a temporal bootstrap variant of \grace{} recovers 9 of 11 causal edges along the Elbe River, for example, with only 1 false positive ($F_1 = 0.86$, AUROC${} = 0.99$), reducing the skeleton's 106 false positives by 99\%.
\end{abstract}

\section{Introduction}
\label{sec:intro}

Inferring causal relationships from observational time series is a central challenge in climate science, neuroscience, finance, gene regulatory network reconstruction, industrial root cause analysis, and many other fields~\citep{peters2017elements}.
Given a multivariate time series $\bx_t = (x_t^1, \ldots, x_t^d)$ with $T$ observations, the goal is to recover the \emph{lagged causal graph} $G$, where $G[c, \ell, e] = 1$ indicates that variable $x^c$ at lag $\ell$ is a direct cause of variable $x^e$ at time $t$.

Two dominant families of methods have emerged for time-series causal discovery, each with complementary strengths and limitations.
\emph{Score-based methods}---Neural Granger Causality~\citep{tank2021neural}, NAVAR~\citep{bussmann2021neural}, DYNOTEARS~\citep{pamfil2020dynotears}, CUTS+~\citep{cheng2023cuts_plus}---take a global optimization view, searching for the graph that best fits the data under a scoring criterion.
They can capture nonlinear dependencies and scale to moderate dimensionality, but enforce sparsity through continuous penalties ($L_1$ or Gumbel-Softmax), requiring post-hoc weight thresholding to obtain a binary graph---and model misspecification or poor scoring assumptions can silently degrade the recovered graph.
Attention-based methods like TCDF~\citep{nauta2019causal} similarly optimize a predictive objective but use attention weights as causal proxies; however, softmax normalization creates artificial competition among candidate parents, and attention thresholds must be dataset-tuned.
\emph{Constraint-based methods}---PCMCI/PCMCI+~\citep{runge2019detecting, runge2020discovering}, CDNOTS~\citep{fesanghary2025cdnots}, SyPI+~\citep{fesanghary2025sypi}, LPCMCI~\citep{gerhardus2020high}---build causal graphs from conditional independence (CI) tests, offering interpretable, statistically principled inference without requiring a global scoring function.
However, these tests are sensitive to sample size and conditioning-set dimension: nonlinear CI tests such as CMI~\citep{runge2018conditional}, KCIT~\citep{zhang2011kernel} ($O(n^3)$), and RCoT~\citep{strobl2019approximate} ($O(d_f^2 n)$ for $d_f$ Fourier features) remain costly at high $d$, forcing practitioners to rely on fast linear tests that miss nonlinear dependencies.
For broader surveys, see \citet{assaad2022survey} and \citet{gong2023causal}.

In this work, we propose \grace{} (Gated Refinement for Accurate Causal Edge discovery), a framework for improving constraint-based causal discovery at high dimensionality.
\grace{} operates in two stages: first, a high-recall constraint-based method such as CDNOTS~\citep{fesanghary2025cdnots} or PCMCI~\citep{runge2019detecting} produces a candidate skeleton; then, a gated neural model refines this skeleton by assigning each candidate edge an independent Hard Concrete gate~\citep{louizos2018learning, maddison2017concrete} trained with $L_0$ regularization, whose gate values concentrate near 0 or 1, yielding a robust binary graph at a natural 0.5 threshold.
The gated model learns which skeleton edges genuinely improve prediction, pruning false positives while preserving true causal links---narrowing the gap with expensive nonlinear CI tests at a fraction of the computational cost.
Although our experiments focus on lagged effects ($\tau \geq 1$), the framework natively supports contemporaneous (lag-0) edges: enabling lag-0 gates allows the same $L_0$ mechanism to select among instantaneous candidates provided by the skeleton.

Our contributions are:
\begin{enumerate}
\item The \grace{} framework: A gated refinement stage that converts a high-recall skeleton into a high-precision causal graph via Hard Concrete gates with $L_0$ regularization. The gating mechanism is independent of the skeleton source: we demonstrate it with both CDNOTS and PCMCI skeletons.
\item An empirically derived regularization rule (Eq.~\ref{eq:auto_lambda}) that adapts $\lambda$ to problem dimensions ($d$, $T$) and skeleton density, paired with $L_0$ normalization that decouples regularization from batch size.
\item Real-world evaluation on the Elbe River network~\citep{stein2025causalrivers}, where a temporal bootstrap variant---applying \grace{} to random temporal windows and retaining edges that appear consistently---handles non-stationarity, hidden confounders, and variable lags without domain-specific heuristics.
\end{enumerate}


\section{Methodology}
\label{sec:method}

\subsection{Problem Formulation}
\label{sec:problem}

Consider a $d$-variate stationary time series $\{\bx_t\}_{t=1}^T$ with $\bx_t \in \RR^d$.
We assume a structural causal model with maximum lag $L$:
\begin{equation}
  x_t^e = f_e\!\left(\{x_{t-\ell}^c : G[c, \ell, e] = 1\}\right) + \epsilon_t^e, \quad \epsilon_t^e \sim \mathcal{N}(0, \sigma_e^2),
  \label{eq:scm}
\end{equation}
where $G \in \{0,1\}^{d \times (L+1) \times d}$ is the lagged causal graph and $f_e$ is the (possibly nonlinear) causal mechanism for effect variable $e$.
The goal is to recover $G$ from observations $\{\bx_t\}_{t=1}^T$.

The gated refinement stage assumes:
(i) \emph{Causal sufficiency}: there are no unobserved common causes;
(ii) \emph{Finite lag}: all causal effects occur within lags $0, \ldots, L$;
(iii) \emph{Additive structure}: effects combine additively across parents, though each individual effect $f_{c,\ell \to e}$ can be nonlinear.
Additional assumptions (e.g., stationarity) are inherited from the skeleton discovery algorithm; when using CDNOTS, nonstationarity is handled natively.

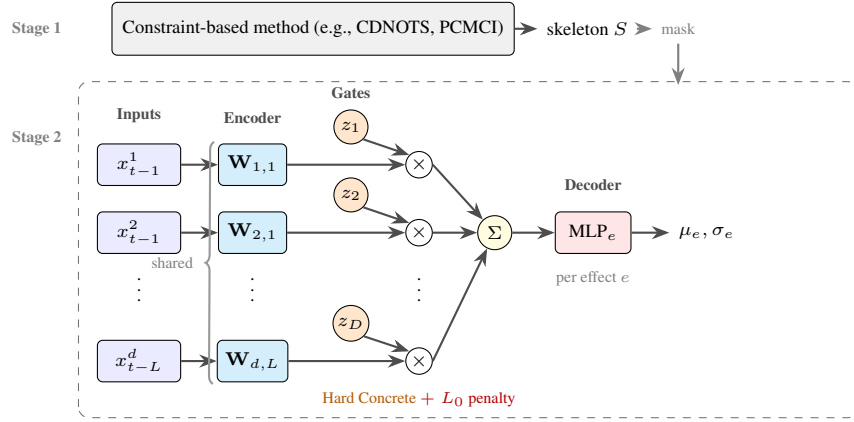
\begin{figure}[t]
\centering
\begin{tikzpicture}[
    >=Stealth,
    box/.style={draw, rounded corners=2pt, minimum height=0.55cm, minimum width=1.1cm,
                font=\scriptsize, inner sep=3pt, fill=#1},
    box/.default=white,
    stagebox/.style={draw, rounded corners=2pt, minimum height=0.7cm, minimum width=1.6cm,
                font=\scriptsize, inner sep=4pt, fill=#1},
    stagebox/.default=white,
    gate/.style={draw, circle, minimum size=0.45cm, font=\scriptsize, inner sep=1pt, fill=orange!20},
    mult/.style={draw, circle, minimum size=0.35cm, font=\scriptsize, inner sep=0pt},
    arr/.style={->, thick, color=black!70},
]

\node[stagebox=gray!12, minimum width=4.2cm] (ci) at (2.3, 1.8) {Constraint-based method (e.g., CDNOTS, PCMCI)};
\node[font=\scriptsize, right=0.3cm of ci] (skel_out) {skeleton $S$};
\draw[arr] (ci) -- (skel_out);

\node[font=\tiny, color=black!50, right=0.2cm of skel_out] (mask_label) {mask};
\draw[arr, color=black!50] (skel_out) -- (mask_label);
\draw[->, thick, color=black!50] (mask_label.south) -- ++(0, -0.55);

\node[font=\tiny\bfseries, color=black!50, anchor=east] at (-0.9, 1.8) {Stage 1};
\node[font=\tiny\bfseries, color=black!50, anchor=east] at (-0.9, 0.35) {Stage 2};


\def\rowA{0}       
\def\rowB{-0.9}    
\def\rowD{-2.6}    

\def\colIn{0}
\def\colEnc{1.5}
\def\colGate{2.8}
\def\colMult{3.7}
\def\colSum{4.7}
\def\colDec{6.0}
\def\colOut{7.5}

\node[box=blue!8] (x1) at (\colIn, \rowA) {$x^1_{t-1}$};
\node[box=blue!8] (x2) at (\colIn, \rowB) {$x^2_{t-1}$};
\node[font=\scriptsize] at (\colIn, -1.55) {$\vdots$};
\node[box=blue!8] (xd) at (\colIn, \rowD) {$x^d_{t-L}$};

\node[box=cyan!15, minimum width=0.9cm] (e1) at (\colEnc, \rowA) {$\bW_{1,1}$};
\node[box=cyan!15, minimum width=0.9cm] (e2) at (\colEnc, \rowB) {$\bW_{2,1}$};
\node[box=cyan!15, minimum width=0.9cm] (ed) at (\colEnc, \rowD) {$\bW_{d,L}$};
\node[font=\scriptsize] at (\colEnc, -1.55) {$\vdots$};

\node[gate] (g1) at (\colGate, \rowA+0.5) {$z_1$};
\node[gate] (g2) at (\colGate, \rowB+0.5) {$z_2$};
\node[gate] (gd) at (\colGate, \rowD+0.5) {$z_D$};

\node[mult] (m1) at (\colMult, \rowA) {$\times$};
\node[mult] (m2) at (\colMult, \rowB) {$\times$};
\node[mult] (md) at (\colMult, \rowD) {$\times$};
\node[font=\scriptsize] at (\colMult, -1.55) {$\vdots$};

\node[draw, circle, minimum size=0.45cm, font=\scriptsize, inner sep=0pt,
      fill=yellow!15] (sum) at (\colSum, \rowB) {$\Sigma$};

\node[box=red!10, minimum width=1.0cm] (dec) at (\colDec, \rowB) {$\text{MLP}_e$};

\node[font=\scriptsize] (out) at (\colOut, \rowB) {$\mu_e, \sigma_e$};

\draw[arr] (x1) -- (e1);
\draw[arr] (x2) -- (e2);
\draw[arr] (xd) -- (ed);

\draw[arr] (e1) -- (m1);
\draw[arr] (e2) -- (m2);
\draw[arr] (ed) -- (md);

\draw[arr] (g1.south east) -- (m1.north west);
\draw[arr] (g2.south east) -- (m2.north west);
\draw[arr] (gd.south east) -- (md.north west);

\draw[arr] (m1.east) -- (sum);
\draw[arr] (m2.east) -- (sum);
\draw[arr] (md.east) -- (sum);

\draw[arr] (sum) -- (dec);
\draw[arr] (dec) -- (out);

\node[above=0.15cm of x1, font=\tiny\bfseries, color=black!70] {Inputs};
\node[above=0.15cm of e1, font=\tiny\bfseries, color=black!70] {Encoder};
\node[above=0.02cm of g1, font=\tiny\bfseries, color=black!70] {Gates};
\node[above=0.15cm of dec, font=\tiny\bfseries, color=black!70] {Decoder};

\draw[decorate, decoration={brace, amplitude=3pt, mirror}, thick, color=black!40]
    ($(e1.north west)+(-0.05,0.02)$) -- ($(ed.south west)+(-0.05,-0.02)$)
    node[midway, left=4pt, font=\tiny, color=black!50, align=center] {shared};

\node[below=0.1cm of dec, font=\tiny, color=black!50] {per effect $e$};

\node[font=\tiny, color=orange!70!black] (hc) at (\colMult, \rowD-0.5)
    {Hard Concrete {\color{red!70!black}$+ \; L_0$ penalty}};

\draw[dashed, rounded corners=4pt, color=black!55]
    (-0.8, -3.35) rectangle (9.5, 1.1);

\end{tikzpicture}
\caption{\grace{} two-stage pipeline. \textbf{Stage~1}: A constraint-based method (e.g., CDNOTS or PCMCI) identifies candidate edges. \textbf{Stage~2}: A gated model refines the skeleton---only edges in $S$ have learnable Hard Concrete gates; all others are masked to zero. The $L_0$ penalty drives gates toward exact zero or one, producing a bimodal distribution that makes the $0.5$ decision boundary robust.}
\label{fig:architecture}
\end{figure}

\subsection{Gated Causal Model}
\label{sec:architecture}

We construct a predictive model for each effect variable $e$, where candidate causal inputs are individually gated.
The architecture consists of three components (Figure~\ref{fig:architecture}): shared encoder, per-effect gated aggregation, and per-effect decoder.

\paragraph{Shared encoder.}
A linear projection maps each lagged input to a hidden representation shared across all effect variables:
\begin{equation}
  h_{c,\ell} = \bW_{c,\ell}\, x_{t-\ell}^c + b_{c,\ell} \in \RR^H,
  \label{eq:encoder}
\end{equation}
where $H$ is the encoder hidden dimension, $\bW_{c,\ell} \in \RR^{H \times 1}$, and $b_{c,\ell} \in \RR^H$.
Sharing the encoder across effects reduces parameters from $O(d^2 L H)$ to $O(d L H)$; effect-specificity is provided by the gates.

\paragraph{Gated aggregation.}
For each effect $e$, a Hard Concrete gate $z_{c,\ell,e} \in \{0, 1\}$ independently controls each candidate parent:
\begin{equation}
  \hat{\bx}_e = \sum_{c,\ell} z_{c,\ell,e} \cdot h_{c,\ell}.
  \label{eq:aggregate}
\end{equation}
Gates are parameterized by learnable log-odds $\log\alpha_{c,\ell,e}$ (Section~\ref{sec:hardconcrete}).

\paragraph{Per-effect decoders.}
Each effect variable has an independent two-layer MLP decoder with SiLU activations that maps the aggregated representation to Gaussian parameters:
\begin{equation}
  (\mu_e, \log\sigma_e) = \text{MLP}_e(\hat{\bx}_e), \quad \text{MLP}_e : \RR^H \to \RR^2.
  \label{eq:decoder}
\end{equation}
Using per-effect decoders (rather than a shared decoder) prevents conflation of distinct causal mechanisms---e.g., a quadratic effect $x^c_{t-1} \to x^e_t$ and a linear effect $x^c_{t-1} \to x^{e'}_t$ can learn separate response functions.

\subsection{Hard Concrete Gates}
\label{sec:hardconcrete}

We adopt the Hard Concrete distribution~\citep{louizos2018learning}, which stretches the binary concrete (Gumbel-Softmax)~\citep{maddison2017concrete, jang2017categorical} distribution to produce exact zeros and ones.

\paragraph{Training (stochastic).}
During training, gates are sampled via the reparameterization trick:
\begin{align}
  u &\sim \text{Uniform}(\epsilon, 1-\epsilon), \label{eq:hc_u} \\
  s &= \sigmoid\!\left(\frac{\log u - \log(1-u) + \log\alpha}{\tau}\right), \label{eq:hc_s} \\
  \bar{z} &= s \cdot (\zeta - \gamma) + \gamma, \label{eq:hc_stretch} \\
  z &= \min(\max(\bar{z}, 0), 1), \label{eq:hc_clamp}
\end{align}
where $\tau = 2/3$ is the temperature, and $\gamma = -0.1$, $\zeta = 1.1$ define the stretch interval.
The stretch beyond $[0,1]$ followed by clamping places positive probability mass on exact 0 and exact 1.

\paragraph{Evaluation (deterministic).}
At test time, we use the deterministic Hard Concrete estimator from \citet{louizos2018learning}:
\begin{equation}
  z = \min\!\left(\max\!\left(\sigmoid(\log\alpha) \cdot (\zeta - \gamma) + \gamma,\; 0\right),\; 1\right).
  \label{eq:hc_det}
\end{equation}
Crucially, the stretch-and-clamp mechanism places positive probability mass on \emph{exact} zero, so trained gates converge to either $z=0$ (edge absent) or $z \approx 1$ (edge present); Appendix~\ref{app:gate_hist} confirms this empirically with gate value histograms showing clear bimodal separation between true and false edges.
The final causal graph is obtained by thresholding deterministic gate values at $0.5$---a natural boundary that falls in the gap between the two modes, making the binary decision robust to the exact threshold choice.
Unlike the \emph{post-hoc weight thresholding} required by $L_1$ and attention-based methods---where continuous edge scores cluster in a narrow range and results are sensitive to the chosen cutoff---the bimodal gate distribution makes the $0.5$ boundary a principled default rather than an arbitrary tuning parameter.
The gate initialization ($\log\alpha={-}0.5$) and $L_0$ penalty strength ($\lambda$) remain as design choices, but these are set once before training rather than tuned after inspecting learned weights.

\paragraph{$L_0$ penalty.}
The probability that a gate is nonzero has a closed-form expression:
\begin{equation}
  \PP(z \neq 0) = \sigmoid\!\left(\log\alpha - \tau \log\!\frac{-\gamma}{\zeta}\right).
  \label{eq:l0_penalty}
\end{equation}
This enables an analytic $L_0$ regularizer without Monte Carlo estimation.

\paragraph{Initialization.}
We initialize $\log\alpha = -0.5$ for all gates, yielding a deterministic gate value of approximately $0.36$ (below the active threshold of $0.5$).
Gates start mostly closed, requiring edges to demonstrate predictive value to open.
This asymmetric initialization reduces false positives compared to a symmetric initialization near $\log\alpha = 0$.

\subsection{Loss Function}
\label{sec:loss}

The training objective combines prediction quality with sparsity:
\begin{align}
  \mathcal{L} &= \underbrace{\frac{1}{Bd}\sum_{t=1}^B \sum_{e=1}^d \text{NLL}(x_t^e \mid \mu_t^e, \sigma_t^e)}_{\text{prediction}} + \underbrace{\lambda \cdot s \sum_{c,\ell,e} \PP(z_{c,\ell,e} \neq 0)}_{\text{$L_0$ sparsity}} \notag \\
  &\quad + \underbrace{\lambda \cdot \lambda_{\text{lag}} \sum_{c,e} \ReLU\!\left(\sum_\ell \PP(z_{c,\ell,e} \neq 0) - k\right)}_{\text{lag concentration (optional)}},
  \label{eq:loss}
\end{align}
where $\text{NLL}(x \mid \mu, \sigma) = \frac{(x-\mu)^2}{2\sigma^2} + \log\sigma + \frac{1}{2}\log 2\pi$ is the Gaussian negative log-likelihood, $B$ is the batch size, and $d$ is the number of variables.

\paragraph{$L_0$ normalization.}
The scaling factor $s = 1/\lceil N/B \rceil$ normalizes the $L_0$ penalty by the number of gradient steps per epoch, ensuring that $\lambda$ values transfer reliably across dataset sizes (Section~\ref{sec:practical}).

\subsection{Practical Considerations}
\label{sec:practical}

\paragraph{Batch size and $L_0$ normalization.}
With $N$ samples and batch size $B$, there are $\lceil N/B \rceil$ gradient steps per epoch.
Without normalization, the $L_0$ penalty is applied per step, while the NLL is averaged per sample, creating an implicit coupling: $\text{effective } \lambda \propto \lambda \cdot N/B$.
Our $L_0$ normalization (factor $s = 1/\lceil N/B \rceil$ in Eq.~\ref{eq:loss}) decouples regularization from batch size.
The batch size is set automatically as $B = \max(32, \min(N/8, 256))$.

\paragraph{Lag concentration (optional).}
By default, each gate is penalized independently by the $L_0$ term and the skeleton determines the lag structure---if the CI skeleton identifies edges at multiple lags for a pair, all corresponding gates are learned without additional constraint.
When domain knowledge or exploratory analysis (e.g., partial cross-correlation) suggests that causal effects concentrate at a small number of lags, an optional lag concentration penalty (third term in Eq.~\ref{eq:loss}) can be enabled by setting $\lambda_{\text{lag}} > 0$.
This penalty discourages more than $k$ active lags per (cause, effect) pair via a soft hinge: $k$ lags are ``free''; the penalty is linear above $k$.
The lag term omits the per-step normalization factor $s$: it is a structural penalty on edge multiplicity, not a per-sample regularizer.
In our experiments, $\lambda_{\text{lag}} = 0$ (disabled) since the skeleton already provides the appropriate lag structure.

\paragraph{Training details.}
We use Adam~\citep{kingma2015adam} with learning rate $10^{-3}$ and weight decay $10^{-4}$.
The encoder hidden dimension is $H = 64$; each per-effect decoder is a two-layer MLP with hidden size 64 and SiLU activations.
The noise scale $\sigma_e$ is not a fixed parameter but a learned output of the decoder (via softplus), optimized jointly with all other parameters.
Training runs for up to 150 epochs with early stopping (patience 20, monitoring training loss).

\subsection{CDNOTS-Gated Refinement}
\label{sec:cdnots_gated}

Constraint-based methods like CDNOTS~\citep{fesanghary2025cdnots} with fast linear CI tests (partial correlation) achieve high recall but suffer from low precision as dimensionality increases: in our benchmarks (Section~\ref{sec:experiments}), CDNOTS achieves $95\%$ TPR but only $37\%$ precision at $d=100$/$T=1000$, yielding F1\,$=$\,0.53.
\grace{} uses the gated model from Section~\ref{sec:architecture} as a \emph{precision filter}: starting from this high-recall skeleton, it learns which edges genuinely improve prediction and prunes the rest.
\paragraph{Pipeline.}
The CDNOTS-Gated refinement operates in two stages:
\begin{enumerate}
\item \textbf{CDNOTS skeleton}: Run CDNOTS with partial correlation CI tests at significance level $\alpha = 0.05$ to obtain a binary skeleton $S \in \{0,1\}^{d \times d \times (L+1)}$.
This produces a high-recall superset of the true causal graph.
\item \textbf{Gated refinement}: Train a single gated model with the skeleton as a hard mask---only edges present in $S$ have learnable gates; all others are fixed to zero.
Gates are initialized at $\log\alpha = -0.5$ (deterministic gate value $\approx 0.36$, below the active threshold of $0.5$) for skeleton edges, so each edge must demonstrate predictive value to open.
\end{enumerate}

\paragraph{$\lambda$ selection.}
We use an empirically derived rule based on problem dimensions and skeleton density:
\begin{equation}
  \lambda = \max\!\bigl(0.007 + 0.16\,\rho_S,\;\; \tfrac{1}{d} + \tfrac{4}{T} - 0.4\,\rho_S^{0.8}\bigr),
  \label{eq:auto_lambda}
\end{equation}
where $\rho_S = |S| / (d^2(L+1) - d)$ is the fraction of candidate edges in the skeleton.
The coefficients were fitted by grid search over the oracle-best $\lambda$ on a held-out calibration set of ${\sim}1{,}000$ random Erd\H{o}s--R\'enyi graphs spanning $d \in \{10, \ldots, 100\}$, $T \in \{500, \ldots, 2000\}$, and edge densities from sparse to 0.12---deliberately disjoint from the SCP benchmarks used in Section~\ref{sec:experiments}, to avoid overfitting to the evaluation data.
The main term ($1/d + 4/T - 0.4\,\rho_S^{0.8}$) adapts to problem scale: more variables and more data allow weaker regularization, with a sublinear density correction.
The floor ($0.007 + 0.16\,\rho_S$) ensures sufficient regularization for dense skeletons.
This reduces computation to a single model-training run.

\paragraph{Why it works.}
The combination succeeds because the two components address complementary failure modes.
CDNOTS with partial correlation provides a computationally cheap skeleton with very high recall---in practice, most nonlinear causal relationships also induce some linear association that partial correlation can detect---but many false positives at high $d$.
The gated model, constrained to the skeleton, operates in a much smaller search space ($|S| \ll d^2 L$), enabling it to learn nonlinear causal mechanisms and prune edges that don't genuinely improve prediction.

A natural question is why a single $\lambda$ value works across all edges, rather than requiring per-edge tuning.
To build intuition (not a formal guarantee for the nonlinear case), Appendix~\ref{app:lambda} analyzes a simplified linear Gaussian model and shows that each edge has a \emph{critical penalty} $\lambda^*$ proportional to its NLL reduction $\Delta$---the improvement in prediction from opening that gate.
True causal edges reduce NLL by a constant $\Delta^{\text{true}} = \Theta(1)$ (determined by signal strength), while false edges can only overfit noise, giving $\Delta^{\text{false}} = O(1/T)$.
The ratio of their critical penalties is therefore $\lambda^{*\text{true}} / \lambda^{*\text{false}} = T \cdot \text{SNR}$, where $\text{SNR}$ is the per-edge signal-to-noise ratio (formally defined in Appendix~\ref{app:lambda} as the conditional SNR).
This means a \emph{separation interval} of valid $\lambda$ values exists and widens linearly with sample size, so any $\lambda$ in this interval recovers the correct graph.
The analysis does not directly derive the constants in Eq.~\ref{eq:auto_lambda}, but it explains three qualitative properties that the heuristic satisfies:
(i)~the skeleton pre-filters candidates, reducing the number of false edges that $\lambda$ must suppress;
(ii)~the gate initialization below the decision boundary creates inertia against false edges opening;
and (iii)~the dependence on $\rho_S$ is justified because denser skeletons admit more false candidates, requiring stronger regularization to keep the expected false positive count bounded.

\paragraph{Comparison with nonlinear CI tests.}
An alternative to gated refinement is using expensive nonlinear CI tests (RCoT, KCIT) within CDNOTS.
Using a nonlinear CI test (RCoT) within the skeleton generally improves F1, since it can detect dependencies that linear partial correlation misses.
However, the runtime cost is steep---in our experiments, 40--75$\times$ slower at $d \geq 50$.
Gated refinement with a fast linear skeleton narrows the F1 gap at a fraction of the cost, requiring only $O(|S| \cdot H)$ per gradient step.
We validate this empirically in Section~\ref{sec:rcot_comparison} (full results in Appendix~\ref{app:rcot}).

\section{Experiments}
\label{sec:experiments}

We evaluate \grace{} on synthetic benchmarks with known ground-truth causal graphs, focusing on high-dimensional scaling ($d \in \{20, 30, 50, 70, 100\}$) with varying sample sizes ($T \in \{300, 500, 1000, 2000\}$).

\subsection{Setup}
\label{sec:setup}

\paragraph{Data generation.}
We generate random structural causal processes (SCPs) with controlled properties---the standard benchmark class for time-series causal discovery, modeling systems common in climate science, econometrics, and neuroscience~\citep{runge2019detecting, assaad2022survey}.
Graphs are \emph{sparse} ($\sim$1.3--2.2 edges per variable), with lags drawn uniformly from $1$ to $L=5$.
Dependency coefficients are sampled uniformly from $[0.40, 0.89]$ or $[-0.89, -0.40]$ with equal probability, and functions are drawn from a mixed pool of linear and nonlinear forms: $\{x, x^2, |x|, \sin(x)\}$.
For each dimensionality $d$, we generate 10 random graphs (seeds 0--9) at $T_{\max} = 2500$ and slice to each target $T$ for evaluation, ensuring that results across sample sizes reflect the same underlying causal structure.

\paragraph{Metrics.}
We report F1 score (harmonic mean of precision and recall), along with precision and true positive rate (TPR/recall) to characterize the precision--recall tradeoff.
Structural Hamming Distance (SHD) results are provided in Appendix~\ref{app:shd}.
All results are averaged over 10 random graph seeds; we report mean $\pm$ standard deviation.

\subsection{Baselines}
\label{sec:baselines}

We compare against representative methods from each major paradigm---constraint-based (PCMCI, CDNOTS), attention-based (TCDF), and score-based (CUTS+).
CUTS+ subsumes the score-based family: it outperforms NGC, DYNOTEARS, and NAVAR on its own benchmarks~\citep{cheng2023cuts_plus}, so including all four would be redundant.
\begin{itemize}
\item \textbf{PCMCI}~\citep{runge2019detecting}: Constraint-based method with partial correlation CI test. Significance level $\alpha = 0.05$.
\item \textbf{CDNOTS}~\citep{fesanghary2025cdnots}: Nonstationary causal discovery with partial correlation.
\item \textbf{TCDF}~\citep{nauta2019causal}: Attention-based convolutional approach with permutation testing.
\item \textbf{CDNOTS-Gated} (\grace{}): Our CDNOTS-Gated refinement pipeline (Section~\ref{sec:cdnots_gated}).
\item \textbf{CUTS+}~\citep{cheng2023cuts_plus}: Score-based method using Gumbel-Softmax gates on a GRU-based message-passing network, designed explicitly for \emph{high-dimensional} causal discovery from irregular time series.
  CUTS+ is the most directly comparable score-based baseline at scale; since it outputs pairwise probabilities without lag resolution, we evaluate both methods at \emph{pair-level} F1 (lags collapsed) for a fair comparison, setting aside \grace{}'s lag-specific output. \grace{} substantially outperforms CUTS+ across all configurations (Appendix~\ref{app:cuts_plus}).
  We use their published default settings with threshold~$0.5$ and set the input window to match the true maximum lag.
\end{itemize}

\subsection{Results: F1 Across Dimensionality and Sample Size}
\label{sec:results}

\begin{table}[t]
\centering
\caption{F1 scores on sparse SCP benchmarks across dimensionality and sample size (mean $\pm$ std over 10 seeds). Best result per row in \textbf{bold}.}
\label{tab:f1_main}
\small
\begin{tabular}{l l cccc}
\toprule
$d$ & \textbf{Method} & $T=300$ & $T=500$ & $T=1000$ & $T=2000$ \\
\midrule
\multirow{4}{*}{20}
& PCMCI         & .436\tiny{$\pm$.03} & .433\tiny{$\pm$.02} & .455\tiny{$\pm$.04} & .454\tiny{$\pm$.03} \\
& CDNOTS       & .785\tiny{$\pm$.04} & .810\tiny{$\pm$.04} & .844\tiny{$\pm$.04} & .856\tiny{$\pm$.06} \\
& TCDF          & .556\tiny{$\pm$.16} & .623\tiny{$\pm$.17} & .660\tiny{$\pm$.18} & .670\tiny{$\pm$.17} \\
& \textbf{CDNOTS-Gated} & \textbf{.794}\tiny{$\pm$.05} & \textbf{.845}\tiny{$\pm$.04} & \textbf{.868}\tiny{$\pm$.04} & \textbf{.893}\tiny{$\pm$.03} \\
\midrule
\multirow{4}{*}{30}
& PCMCI         & .341\tiny{$\pm$.02} & .346\tiny{$\pm$.03} & .354\tiny{$\pm$.03} & .356\tiny{$\pm$.03} \\
& CDNOTS       & .768\tiny{$\pm$.04} & .784\tiny{$\pm$.04} & .816\tiny{$\pm$.06} & .830\tiny{$\pm$.03} \\
& TCDF          & .518\tiny{$\pm$.12} & .608\tiny{$\pm$.14} & .683\tiny{$\pm$.15} & .688\tiny{$\pm$.15} \\
& \textbf{CDNOTS-Gated} & \textbf{.804}\tiny{$\pm$.04} & \textbf{.851}\tiny{$\pm$.05} & \textbf{.893}\tiny{$\pm$.03} & \textbf{.908}\tiny{$\pm$.04} \\
\midrule
\multirow{4}{*}{50}
& PCMCI         & .221\tiny{$\pm$.02} & .228\tiny{$\pm$.02} & .236\tiny{$\pm$.02} & .241\tiny{$\pm$.02} \\
& CDNOTS       & .717\tiny{$\pm$.05} & .720\tiny{$\pm$.04} & .742\tiny{$\pm$.04} & .722\tiny{$\pm$.04} \\
& TCDF          & .431\tiny{$\pm$.08} & .600\tiny{$\pm$.10} & .691\tiny{$\pm$.09} & .724\tiny{$\pm$.08} \\
& \textbf{CDNOTS-Gated} & \textbf{.841}\tiny{$\pm$.03} & \textbf{.886}\tiny{$\pm$.03} & \textbf{.915}\tiny{$\pm$.02} & \textbf{.911}\tiny{$\pm$.03} \\
\midrule
\multirow{4}{*}{70}
& PCMCI         & .161\tiny{$\pm$.02} & .168\tiny{$\pm$.02} & .171\tiny{$\pm$.02} & .177\tiny{$\pm$.02} \\
& CDNOTS       & .651\tiny{$\pm$.04} & .665\tiny{$\pm$.03} & .666\tiny{$\pm$.05} & .648\tiny{$\pm$.05} \\
& TCDF          & .322\tiny{$\pm$.04} & .529\tiny{$\pm$.07} & .724\tiny{$\pm$.07} & .768\tiny{$\pm$.06} \\
& \textbf{CDNOTS-Gated} & \textbf{.849}\tiny{$\pm$.05} & \textbf{.896}\tiny{$\pm$.03} & \textbf{.927}\tiny{$\pm$.02} & \textbf{.929}\tiny{$\pm$.02} \\
\midrule
\multirow{4}{*}{100}
& PCMCI         & .105\tiny{$\pm$.02} & .113\tiny{$\pm$.02} & .117\tiny{$\pm$.02} & .119\tiny{$\pm$.02} \\
& CDNOTS       & .571\tiny{$\pm$.04} & .559\tiny{$\pm$.06} & .531\tiny{$\pm$.07} & .513\tiny{$\pm$.07} \\
& TCDF          & .269\tiny{$\pm$.04} & .504\tiny{$\pm$.10} & .718\tiny{$\pm$.10} & .785\tiny{$\pm$.09} \\
& \textbf{CDNOTS-Gated} & \textbf{.830}\tiny{$\pm$.06} & \textbf{.878}\tiny{$\pm$.04} & \textbf{.920}\tiny{$\pm$.02} & \textbf{.922}\tiny{$\pm$.02} \\
\bottomrule
\end{tabular}
\end{table}

\begin{figure}[t]
\centering
\includegraphics[width=\textwidth]{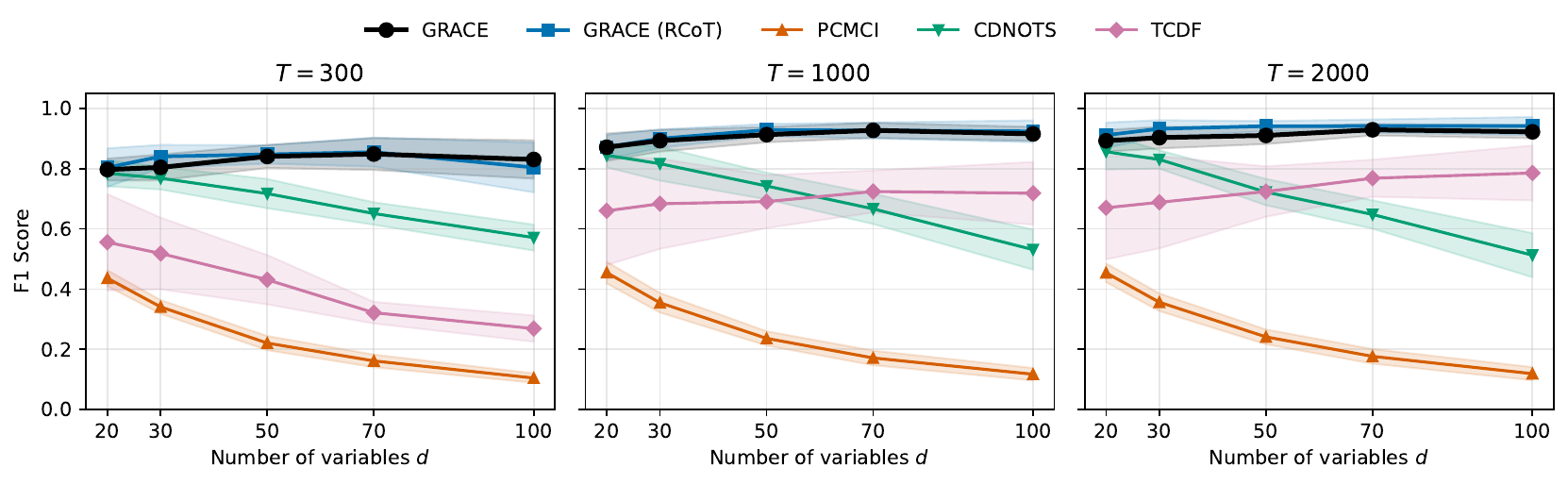}
\caption{$F_1$ score vs.\ number of variables $d$ at three sample sizes. CDNOTS-Gated (black) consistently outperforms all baselines, and the advantage grows with dimensionality.}
\label{fig:f1_vs_d}
\end{figure}

Table~\ref{tab:f1_main} and Figure~\ref{fig:f1_vs_d} present F1 scores across all methods, dimensionalities, and sample sizes.
We highlight several key findings:

\paragraph{CDNOTS-Gated dominates at high dimensionality.}
At $d=50$/$T=1000$, CDNOTS-Gated achieves F1\,$=$\,0.915, compared to 0.742 for CDNOTS and 0.691 for TCDF.
At $d=100$, the gap widens further (F1\,$=$\,0.878 vs.\ 0.559 at $T=500$), with the advantage growing monotonically in $d$.
PCMCI achieves $>97\%$ recall at all $d$ but precision collapses from 30\% ($d=20$) to 6\% ($d=100$)---a common challenge as the hypothesis space grows as $O(d^2 L)$.
TCDF achieves high precision ($>93\%$) through permutation testing but limited recall (${\sim}56\%$), plateauing around F1\,$\approx$\,0.72.
CUTS+ precision collapses at high $d$: pair-level F1 drops to 0.055 at $d=100$ (Appendix~\ref{app:cuts_plus}).

\begin{table}[t]
\centering
\caption{Precision and TPR at $T=1000$ across dimensionality (mean over 10 seeds). CDNOTS-Gated preserves CDNOTS's high recall while dramatically improving precision.}
\label{tab:prec_tpr}
\small
\begin{tabular}{l cc cc cc cc}
\toprule
& \multicolumn{2}{c}{$d=20$} & \multicolumn{2}{c}{$d=50$} & \multicolumn{2}{c}{$d=70$} & \multicolumn{2}{c}{$d=100$} \\
\cmidrule(lr){2-3} \cmidrule(lr){4-5} \cmidrule(lr){6-7} \cmidrule(lr){8-9}
\textbf{Method} & Prec & TPR & Prec & TPR & Prec & TPR & Prec & TPR \\
\midrule
PCMCI         & .298 & .973 & .135 & .978 & .094 & .984 & .062 & .982 \\
CDNOTS       & .837 & .860 & .617 & .934 & .514 & .953 & .372 & .951 \\
TCDF          & .938 & .537 & .933 & .556 & .931 & .597 & .954 & .584 \\
\textbf{CDNOTS-Gated} & \textbf{.950} & .804 & \textbf{.967} & .870 & \textbf{.977} & .883 & \textbf{.974} & .872 \\
\bottomrule
\end{tabular}
\end{table}

\paragraph{Precision--recall tradeoff.}
Table~\ref{tab:prec_tpr} reveals the mechanism: CDNOTS-Gated preserves most of CDNOTS's recall (83--88\% vs.\ 86--95\%) while boosting precision from 37--84\% to 86--97\%.
At $d=70$, precision reaches 97.5\%---nearly every predicted edge is correct.
Runtime remains practical (${\sim}4$\,min at $d=100$; Appendix~\ref{app:runtime}).

\paragraph{Nonlinear CI skeleton.}
\label{sec:rcot_comparison}
Replacing the linear skeleton with RCoT~\citep{strobl2019approximate} improves F1 slightly (0.944 vs.\ 0.922 at $d=100$/$T=2000$) but at $75\times$ the cost; at lower $T$, ParCorr-Gated matches or exceeds RCoT-Gated (Appendix~\ref{app:rcot}).

\paragraph{Ablations and generalization.}
Without a skeleton, $L_0$ suppression alone cannot recover the signal at high $d$ (F1 degrades substantially); the adaptive $\lambda$ formula consistently outperforms fixed values; lag overspecification is handled gracefully ($F_1 \geq 0.878$) via $L_0$ suppression (Appendix~\ref{app:ablation}).
\grace{} generalizes across skeleton sources: PCMCI-Gated yields $3$--$5\times$ F1 over raw PCMCI (Appendix~\ref{app:pcmci_gated}), and DYNOTEARS-Gated matches CDNOTS-Gated at $d \leq 50$ (Appendix~\ref{app:dynotears_gated}).
It also generalizes across non-additive mechanisms (Lorenz-96; Appendix~\ref{app:lorenz96}) and diverse graph topologies (scale-free, Erd\H{o}s--R\'enyi, small-world; Appendix~\ref{app:topology}).

\subsection{Real-World Evaluation: Elbe River Network}
\label{sec:causalrivers}

Synthetic benchmarks offer controlled evaluation but cannot capture the compounding assumption violations present in real-world systems.
We evaluate \grace{} on a river network dataset~\citep{stein2025causalrivers} providing water level measurements along river networks in eastern Germany, where ground truth is determined by physical flow direction (upstream $\to$ downstream).

\paragraph{Challenges.}
River data violates standard assumptions: rainfall acts as a time-varying, spatially-varying hidden confounder; causal lags vary with water level (higher flow $=$ faster propagation); seasonal floods create distributional shifts; and reservoirs disconnect natural flow.
These challenges make causal discovery on river data substantially harder than on synthetic benchmarks~\citep{stein2025causalrivers}.

\paragraph{Setup.}
We focus on the Elbe River's main branch in eastern Germany: 12 gauging stations forming a chain graph with 11 directed edges (upstream $\to$ downstream), measured at 6-hour resolution over 2019--2023 ($T=3{,}164$).
We compare PCMCI, CDNOTS, single-run \grace{}, and \grace{}-Bootstrap.
When applied to the full multi-year time series, all skeleton methods achieve perfect recall but extremely poor precision due to ubiquitous confounding---e.g., CDNOTS recovers all 11 true edges but produces 106 false positives ($F_1=0.17$).

\paragraph{Temporal bootstrap.}
We apply \grace{} independently to random 1-month windows and retain edges appearing in more than a $\tau$ fraction of runs---a temporal analogue of stability selection~\citep{meinshausen2010stability}.
Because confounders like rainfall are transient and episodic, they inflate correlations only in the windows they affect; true flow edges, by contrast, persist across all hydrological conditions.
Windowing thus separates stable causal signal from transient confounding without requiring confounder measurements.

\paragraph{Results.}
We set $\tau=0.70$, requiring an edge to appear in more than 70\% of windows to be retained---a conservative threshold that does not require ground-truth calibration and is interpretable on its own terms.
Table~\ref{tab:elbe} shows that \grace{}-Bootstrap (50 runs, $\tau=0.70$) recovers 9 of 11 edges with only 1 false positive ($F_1=0.86$, Precision\,$=$\,0.90), reducing CDNOTS's 106 FPs by 99\%.
The bootstrap AUROC reaches 0.986, showing that edge frequency provides an excellent ranking even when the binary threshold excludes borderline edges.

\begin{table}[t]
\centering
\caption{Elbe River's main branch ($d=12$, 11 true edges, $T=3{,}164$, 6h resolution)~\citep{stein2025causalrivers}. The CI skeleton achieves perfect recall but produces 106 false positives. \grace{}-Bootstrap applies \grace{} to 50 random 1-month windows and thresholds on edge frequency ($\tau=0.70$).}
\label{tab:elbe}
\small
\begin{tabular}{l ccccc}
\toprule
\textbf{Method} & \textbf{AUROC} & \textbf{Prec} & \textbf{Rec} & \textbf{F1} & \textbf{FP} \\
\midrule
CDNOTS (skeleton)          & .839 & .094 & 1.00 & .171 & 106 \\
\grace{} (single run)      & .938 & .138 & 1.00 & .242 & 69 \\
\textbf{\grace{}-Bootstrap} & \textbf{.986} & \textbf{.900} & \textbf{.818} & \textbf{.857} & \textbf{1} \\
\bottomrule
\end{tabular}
\end{table}

Together, the synthetic experiments (Section~\ref{sec:results}) and this real-world evaluation cover complementary axes of difficulty: the former tests scalability under controlled conditions, the latter tests robustness when causal assumptions are simultaneously and severely violated---a combination that no single benchmark can address.

\section{Conclusion}
\label{sec:conclusion}

We presented \grace{}, a framework that refines constraint-based causal discovery via Hard Concrete gates with $L_0$ regularization.
On synthetic benchmarks, \grace{} nearly doubles F1 over raw CDNOTS at $d=100$ (0.92 vs.\ 0.53) with 97\% precision, generalizing across skeleton sources (PCMCI, DYNOTEARS) and graph topologies.
On the real-world Elbe River benchmark~\citep{stein2025causalrivers}, a temporal bootstrap variant achieves $F_1 = 0.86$ with AUROC${} = 0.99$ despite hidden confounders, variable lags, and distributional shifts.

\paragraph{Limitations.}
On dense graphs (scale-free hubs at $d \gtrsim 70$), per-edge SNR decreases and gate values cluster below threshold---AUROC remains high but binary F1 degrades, partially mitigated by more data.
Spurious edges from latent confounders may survive if they improve prediction.
The linear skeleton may miss purely symmetric nonlinear dependencies (e.g., $x^2$) that induce no linear association; a nonlinear CI skeleton mitigates this at higher cost (Appendix~\ref{app:rcot}).
Gated refinement trades a small number of true positives for a large precision gain; the recall reduction is minor in our experiments but may be more pronounced at very low sample sizes.
\bibliographystyle{plainnat}
\bibliography{references}

\appendix

\section{Gate Value Distribution}
\label{app:gate_hist}

Figure~\ref{fig:gate_hist} plots the distribution of deterministic gate values (Eq.~\ref{eq:hc_det}) after training, partitioned by whether each skeleton edge corresponds to a true or false causal link.
At $d=100$, false edges cluster near zero (mass at $0.0$ and $0.25$--$0.40$) while true edges concentrate at $0.55$--$0.85$, with almost no overlap at the $0.5$ decision boundary.
At $d=50$, the separation is somewhat narrower, with minor overlap near $0.5$, consistent with the slightly lower F1 at this scale.
The clear bimodal pattern confirms that Hard Concrete gates learn to separate true from false edges without requiring post-hoc threshold tuning: the $0.5$ boundary falls squarely in the gap between the two modes.

\begin{figure}[ht]
  \centering
  \includegraphics[width=\textwidth]{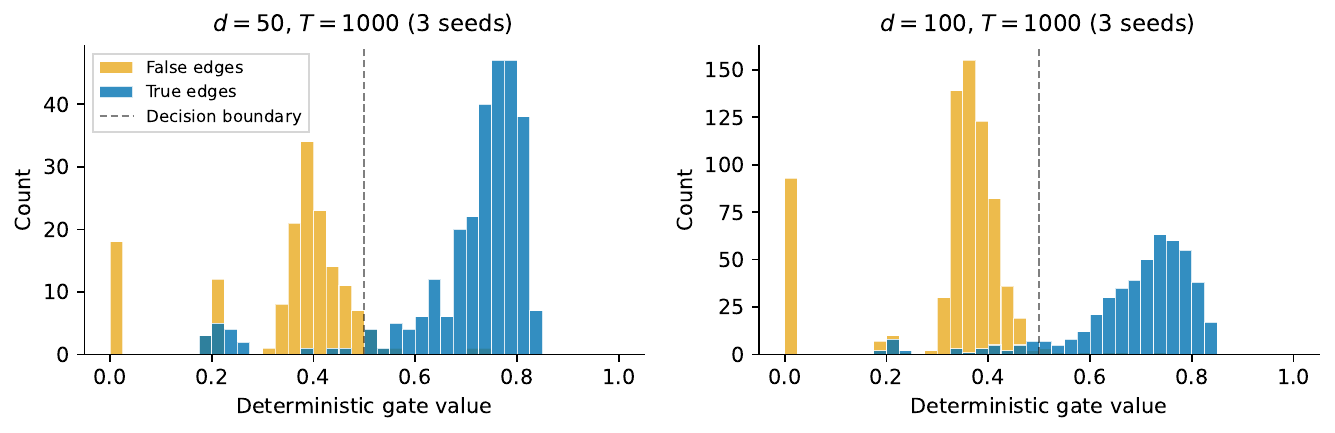}
  \caption{Distribution of deterministic gate values on SCP benchmarks ($T=1000$, 3 seeds).
  True causal edges (blue) concentrate above the $0.5$ decision boundary; false edges (yellow) concentrate below it.
  The separation is sharper at higher $d$, where the $L_0$ penalty provides stronger sparsity pressure.}
  \label{fig:gate_hist}
\end{figure}

\section{Theoretical Analysis of \texorpdfstring{$\lambda$}{lambda}-Sensitivity}
\label{app:lambda}

This appendix analyzes the $\lambda$-sensitivity of the gated model under a simplified linear Gaussian setting.
The results provide \emph{qualitative intuition} for why a single $\lambda$ can separate true and false edges, but do not constitute identifiability guarantees for the full nonlinear model with MLP decoders.
We view these results as motivation for the design choices in Section~\ref{sec:cdnots_gated}, not as formal recovery proofs.

\subsection{Setup and Notation}

Consider a stationary VAR process with $d$ variables and maximum lag $L$:
\begin{equation}
  x^e_t = \sum_{(c,\ell) \in \text{pa}(e)} \beta_{c\ell}^e \, x^c_{t-\ell} + \varepsilon^e_t, \qquad \varepsilon^e_t \stackrel{\text{iid}}{\sim} \mathcal{N}(0, \sigma_e^2),
  \label{eq:dgp}
\end{equation}
where $\text{pa}(e)$ denotes the true parent set of variable $e$ (the set of $(c,\ell)$ pairs with $\beta_{c\ell}^e \neq 0$).
We analyze the linear encoder case; the nonlinear extension is discussed in the Remark (Nonlinear extension) below.

The per-sample loss for effect $e$, with weight $w_{c\ell}^e$ for each gated input, is:
\begin{equation}
  \ell_e = \frac{1}{2\hat\sigma_e^2}\left(x^e_t - \sum_{c,\ell} z_{c\ell e} \, w_{c\ell}^e \, x^c_{t-\ell}\right)^2 + \log\hat\sigma_e + \tfrac{1}{2}\log 2\pi.
  \label{eq:per_sample_loss}
\end{equation}
The full training objective (Eq.~\ref{eq:loss}) averages $\ell_e$ over both $T$ samples and $d$ effects---i.e., $\overline{\text{NLL}} = \frac{1}{d}\sum_e \frac{1}{T}\sum_t \ell_e(t)$---and adds the $L_0$ penalty $\lambda \cdot s \sum_{c,\ell,e} \PP(z_{c\ell e} \neq 0)$, where $s = 1/\lceil T/B \rceil$ is the per-step normalization factor.
The $1/d$ averaging is important: it means each gate's NLL contribution is scaled by $1/d$, which enters the critical $\lambda$ formula below.

\subsection{Per-Edge NLL Reduction}

\begin{definition}[Marginal NLL reduction]
For edge $(c,\ell,e)$, define $\Delta_{c\ell e}$ as the decrease in average NLL when the gate is opened (with optimal weight $w^*$) compared to when it is forced closed:
\begin{equation}
  \Delta_{c\ell e} = \frac{1}{T}\sum_{t=1}^{T}\left[\ell_e\big|_{z_{c\ell e}=0} - \ell_e\big|_{z_{c\ell e}=1,\, w=w^*}\right].
  \label{eq:delta_def}
\end{equation}
\end{definition}

Under the linear Gaussian DGP (Eq.~\ref{eq:dgp}), assuming the other gates and weights are at their population optima:

\paragraph{True edge $((c,\ell) \in \text{pa}(e))$.}
The optimal weight is $w^{*} = \beta_{c\ell}^e$, and the NLL reduction equals the variance uniquely explained by this edge---the variance of $x^c_{t-\ell}$ after projecting out all other parents already included:
\begin{equation}
  \Delta_{c\ell e}^{\text{true}} = \frac{(\beta_{c\ell}^e)^2 \, \widehat{\text{Var}}(x^c_{t-\ell} \mid \widehat{\text{pa}}^e_{-(c,\ell)})}{2\,\sigma_e^2},
  \label{eq:delta_true}
\end{equation}
where $\widehat{\text{Var}}(x^c_{t-\ell} \mid \widehat{\text{pa}}^e_{-(c,\ell)})$ is the empirical variance of the residual of $x^c_{t-\ell}$ after linear projection onto all other true parents of $e$.
This is the conditional variance formula from multivariate regression: in a regression with correlated regressors, the incremental gain from one predictor depends on its residualized variance, not its marginal variance.
By the law of large numbers this converges to $\text{Var}(x^c_{t-\ell} \mid \text{pa}^e_{-(c,\ell)}) > 0$ (positive unless a true parent is a perfect linear combination of others), so $\Delta^{\text{true}}$ converges to a \emph{positive constant}.

\paragraph{False edge $((c,\ell) \notin \text{pa}(e))$.}
Opening a spurious gate adds one free parameter that can only fit noise.
By standard regression theory, the in-sample NLL reduction from fitting one noise regressor to $T$ residuals is:
\begin{equation}
  \Delta_{c\ell e}^{\text{false}} = \frac{\hat{r}_{c\ell e}^2}{2} + O(T^{-1}),
  \label{eq:delta_false}
\end{equation}
where $\hat{r}_{c\ell e}$ is the sample correlation between the residual and $x^c_{t-\ell}$.
Under the null hypothesis (no causal effect), $T \cdot \hat{r}^2 \sim \chi^2(1)$, so $\hat{r}^2$ is itself $O(1/T)$ and
\begin{equation}
  \EE[\Delta^{\text{false}}] = \frac{1}{2T}.
  \label{eq:delta_false_exp}
\end{equation}

\subsection{Critical \texorpdfstring{$\lambda$}{lambda} and Separation}

\begin{proposition}[Edge-wise critical $\lambda$]
\label{prop:critical_lambda}
At a local minimum of the expected loss, gate $(c,\ell,e)$ is open (i.e., $\log\alpha_{c\ell e} > 0$, giving deterministic gate $z > 0.5$) if and only if the per-edge NLL reduction exceeds the marginal $L_0$ cost.
Define the \emph{critical penalty} as:
\begin{equation}
  \lambda^*_{c\ell e} = \frac{\Delta_{c\ell e} \cdot C}{d \cdot s},
  \quad C = \frac{R}{\kappa},
  \label{eq:critical_lambda}
\end{equation}
where $\kappa = \frac{\partial}{\partial \log\alpha}\PP(z \neq 0)\big|_{\log\alpha=0}$ is the sensitivity of the gate-open probability at the decision boundary, and $R = \EE\!\left[\frac{\partial z}{\partial \log\alpha}\right]\big|_{\log\alpha=0}$ is the expected reparameterization gradient of the gate value.
The factor $1/d$ arises because the NLL is averaged over $d$ effects, so each gate's contribution to the loss is scaled by $1/d$.
With the default parameters $(\tau = 2/3, \gamma = -0.1, \zeta = 1.1)$:
\begin{equation}
  \kappa = \sigma(c_0)(1 - \sigma(c_0)) \approx 0.140,
  \quad c_0 = -\tau \log\frac{-\gamma}{\zeta} \approx 1.60,
\end{equation}
and $R \approx 0.219$ (computed via Monte Carlo over the Hard Concrete noise), yielding $C = R/\kappa \approx 1.56$.
The gate is open when $\lambda < \lambda^*_{c\ell e}$ and closed when $\lambda > \lambda^*_{c\ell e}$.
\end{proposition}

\begin{proof}
The training loss decomposes as $\mathcal{L} = \overline{\text{NLL}} + \lambda \cdot s \sum_{c,\ell,e} \PP(z_{c\ell e} \neq 0)$, where $\overline{\text{NLL}} = \frac{1}{d}\sum_e \frac{1}{T}\sum_t \ell_e(t)$.
Since gate $(c,\ell,e)$ only affects $\ell_e$, the stationarity condition is:
\begin{equation}
  \frac{\partial \mathcal{L}}{\partial \log\alpha_{c\ell e}} = \underbrace{-\frac{\Delta_{c\ell e}}{d} \cdot \EE\!\left[\frac{\partial z}{\partial \log\alpha}\right]}_{\text{NLL gradient (favors opening)}} + \underbrace{\lambda \cdot s \cdot \frac{\partial \PP(z \neq 0)}{\partial \log\alpha_{c\ell e}}}_{\text{$L_0$ gradient (favors closing)}} = 0.
  \label{eq:stationarity}
\end{equation}
At the decision boundary $\log\alpha = 0$, both $R$ and $\kappa$ are strictly positive.
Setting the gradients equal and solving for $\lambda$ gives $\lambda^* = \Delta \cdot R / (d \cdot s \cdot \kappa) = \Delta \cdot C / (d \cdot s)$.
\end{proof}

\begin{corollary}[Sample-size separation]
\label{cor:separation}
Substituting the NLL reductions from Eqs.~\ref{eq:delta_true}--\ref{eq:delta_false}, and writing $V^{\text{cond}}_{c\ell e} = \text{Var}(x^c_{t-\ell} \mid \text{pa}^e_{-(c,\ell)})$ for the conditional variance:
\begin{align}
  \lambda^{*\text{true}}_{c\ell e} &= \frac{(\beta_{c\ell}^e)^2 \, V^{\text{cond}}_{c\ell e} \cdot C}{2\,\sigma_e^2 \cdot d \cdot s} + o(1), \label{eq:lam_true} \\[4pt]
  \lambda^{*\text{false}}_{c\ell e} &= \frac{C}{2T \cdot d \cdot s} + o(T^{-1}). \label{eq:lam_false}
\end{align}
The separation ratio is:
\begin{equation}
  \frac{\lambda^{*\text{true}}}{\lambda^{*\text{false}}} = T \cdot \frac{(\beta_{c\ell}^e)^2 \, V^{\text{cond}}_{c\ell e}}{\sigma_e^2} = T \cdot \text{SNR}^{\text{cond}}_{c\ell e}.
  \label{eq:ratio}
\end{equation}
Note that $C$, $d$, and $s$ cancel in the ratio, so the separation depends only on $T$ and the edge's \emph{conditional} signal-to-noise ratio $\text{SNR}^{\text{cond}} = (\beta_{c\ell}^e)^2 V^{\text{cond}}_{c\ell e} / \sigma_e^2$.
For any minimum signal strength $\text{SNR}_{\min} > 0$, there exists $T^* = 1/\text{SNR}_{\min}$ such that for $T > T^*$, the interval $(\lambda^{*\text{false}}_{\max},\, \lambda^{*\text{true}}_{\min})$ is nonempty, and \emph{any} $\lambda$ in this interval yields perfect edge recovery (zero FP, full TP).
This bound uses the per-edge expected $\Delta^{\text{false}}$; for a simultaneous guarantee over all $n_{\text{false}}$ spurious edges, extreme-value corrections on $\max_i \hat{r}_i^2$ increase $T^*$ to $O(\log(n_{\text{false}}) / \text{SNR}_{\min})$.
\end{corollary}

\begin{remark}[Nonlinear extension]
For a nonlinear DGP $x^e_t = f_e(\{x^c_{t-\ell}\}_{(c,\ell) \in \text{pa}(e)}) + \varepsilon^e_t$ with MLP decoders, the NLL reduction for a true edge generalizes to
$\Delta^{\text{true}} = \|\text{proj}_{x^c} r^e\|^2 / (2\sigma_e^2)$,
where $r^e$ is the residual from all other parents and $\text{proj}_{x^c} r^e$ denotes the projection of $r^e$ onto $x^c_{t-\ell}$.
This is the nonlinear analogue of the conditional variance $V^{\text{cond}}_{c\ell e}$ in Eq.~\ref{eq:delta_true}: both measure the unique variance contribution of the $(c,\ell)$ edge after accounting for all other parents.
The false-edge reduction becomes $O(p_{\text{eff}} / T)$, where $p_{\text{eff}}$ is the effective number of parameters the decoder can fit through one open gate---equal to the hidden dimension $h$ when the encoder is pre-trained.
The separation ratio is $T \cdot \text{SNR}^{\text{cond}} / p_{\text{eff}}$, preserving the $O(T)$ growth whenever the true causal signal is bounded away from zero.
\end{remark}

\begin{remark}[Gate initialization effect]
The critical $\lambda$ in Proposition~\ref{prop:critical_lambda} is evaluated at the decision boundary $\log\alpha = 0$.
At the default initialization $\log\alpha_0 = -0.5$, the ratio $C(-0.5) = R(-0.5)/\kappa(-0.5) \approx 1.12 < C(0)$, so the threshold for a false gate to \emph{begin} climbing toward the boundary is lower: $\lambda^*_{\text{start}} = \Delta_{\text{false}} \cdot C(-0.5) / (d \cdot s)$.
For $\lambda > \lambda^*_{\text{start}}$, false gates are pushed further closed from initialization and never reach the decision boundary, providing an additional safety margin beyond the equilibrium analysis.
\end{remark}

\section{Nonlinear CI Test Results}
\label{app:rcot}

Section~\ref{sec:rcot_comparison} compares the fast linear CI skeleton (ParCorr) with a nonlinear one (RCoT) in terms of F1 accuracy and runtime cost.
Here we present the full results.
We compare CDNOTS-Gated with the default ParCorr skeleton against CDNOTS-Gated with an RCoT skeleton, reporting both F1 and runtime.

\begin{table}[ht]
\centering
\caption{F1 scores and runtime (seconds) for ParCorr-Gated vs.\ RCoT-Gated (mean $\pm$ std over 10 seeds). Best F1 per row in \textbf{bold}. RCoT-Gated generally outperforms ParCorr-Gated, but at 20--75$\times$ the runtime cost.}
\label{tab:f1_rcot}
\small
\begin{tabular}{l l cc cc}
\toprule
$d$ & $T$ & \makecell{ParCorr-Gated\\F1} & \makecell{RCoT-Gated\\F1} & \makecell{ParCorr\\Time (s)} & \makecell{RCoT\\Time (s)} \\
\midrule
\multirow{4}{*}{20}
& 300  & .794\tiny{$\pm$.05} & \textbf{.804}\tiny{$\pm$.06} & 38\tiny{$\pm$8} & 380\tiny{$\pm$282} \\
& 500  & .845\tiny{$\pm$.04} & \textbf{.853}\tiny{$\pm$.04} & 42\tiny{$\pm$9} & 726\tiny{$\pm$542} \\
& 1000 & .868\tiny{$\pm$.04} & \textbf{.872}\tiny{$\pm$.04} & 38\tiny{$\pm$9} & 852\tiny{$\pm$688} \\
& 2000 & .893\tiny{$\pm$.03} & \textbf{.911}\tiny{$\pm$.04} & 43\tiny{$\pm$10} & 1{,}129\tiny{$\pm$1{,}018} \\
\midrule
\multirow{4}{*}{30}
& 300  & .804\tiny{$\pm$.04} & \textbf{.840}\tiny{$\pm$.04} & 56\tiny{$\pm$5} & 611\tiny{$\pm$153} \\
& 500  & .851\tiny{$\pm$.05} & \textbf{.879}\tiny{$\pm$.03} & 53\tiny{$\pm$11} & 1{,}189\tiny{$\pm$347} \\
& 1000 & .893\tiny{$\pm$.03} & \textbf{.900}\tiny{$\pm$.03} & 55\tiny{$\pm$13} & 1{,}554\tiny{$\pm$449} \\
& 2000 & .908\tiny{$\pm$.04} & \textbf{.932}\tiny{$\pm$.03} & 61\tiny{$\pm$13} & 1{,}845\tiny{$\pm$635} \\
\midrule
\multirow{4}{*}{50}
& 300  & .841\tiny{$\pm$.03} & \textbf{.847}\tiny{$\pm$.03} & 99\tiny{$\pm$3} & 1{,}874\tiny{$\pm$610} \\
& 500  & .886\tiny{$\pm$.03} & \textbf{.897}\tiny{$\pm$.02} & 95\tiny{$\pm$5} & 3{,}479\tiny{$\pm$1{,}158} \\
& 1000 & .915\tiny{$\pm$.02} & \textbf{.928}\tiny{$\pm$.02} & 100\tiny{$\pm$9} & 4{,}060\tiny{$\pm$1{,}326} \\
& 2000 & .911\tiny{$\pm$.03} & \textbf{.941}\tiny{$\pm$.02} & 105\tiny{$\pm$10} & 4{,}671\tiny{$\pm$1{,}429} \\
\midrule
\multirow{4}{*}{70}
& 300  & \textbf{.849}\tiny{$\pm$.05} & .855\tiny{$\pm$.04} & 129\tiny{$\pm$7} & 4{,}385\tiny{$\pm$1{,}638} \\
& 500  & \textbf{.896}\tiny{$\pm$.03} & .859\tiny{$\pm$.05} & 129\tiny{$\pm$2} & 7{,}926\tiny{$\pm$2{,}743} \\
& 1000 & \textbf{.927}\tiny{$\pm$.02} & .918\tiny{$\pm$.03} & 159\tiny{$\pm$9} & 9{,}287\tiny{$\pm$3{,}118} \\
& 2000 & .929\tiny{$\pm$.02} & \textbf{.943}\tiny{$\pm$.03} & 156\tiny{$\pm$10} & 10{,}074\tiny{$\pm$2{,}954} \\
\midrule
\multirow{4}{*}{100}
& 300  & \textbf{.830}\tiny{$\pm$.06} & .804\tiny{$\pm$.08} & 181\tiny{$\pm$15} & 8{,}321\tiny{$\pm$3{,}745} \\
& 500  & \textbf{.878}\tiny{$\pm$.04} & .875\tiny{$\pm$.06} & 234\tiny{$\pm$11} & 14{,}971\tiny{$\pm$6{,}567} \\
& 1000 & \textbf{.920}\tiny{$\pm$.02} & .916\tiny{$\pm$.04} & 228\tiny{$\pm$8} & 17{,}059\tiny{$\pm$7{,}128} \\
& 2000 & .922\tiny{$\pm$.02} & \textbf{.944}\tiny{$\pm$.03} & 243\tiny{$\pm$21} & 18{,}962\tiny{$\pm$7{,}103} \\
\bottomrule
\end{tabular}

\end{table}

Table~\ref{tab:f1_rcot} compares both F1 and runtime for the two Gated variants.
RCoT-Gated outperforms ParCorr-Gated at high $T$, since RCoT can detect nonlinear conditional dependencies that linear partial correlation misses.
At $d=100$/$T=2000$, RCoT-Gated achieves F1\,$=$\,0.944 vs.\ 0.922 for ParCorr-Gated.
However, at moderate $T$, the gap narrows or reverses (e.g., 0.916 vs.\ 0.920 at $d=100$/$T=1000$), while the runtime cost is substantial: at $d=50$/$T=1000$, RCoT-Gated requires 4{,}060s vs.\ 100s for ParCorr-Gated---a 40$\times$ slowdown.
At $d=100$, the ratio reaches 75$\times$ (${\sim}17{,}000$s vs.\ ${\sim}228$s).
The practical conclusion is clear: for $d \geq 50$, ParCorr-Gated is the most practical choice, with RCoT offering diminishing returns at disproportionate cost.

\section{SHD Results}
\label{app:shd}

We report the Structural Hamming Distance (SHD)---the total number of edge additions, deletions, and reversals needed to transform the predicted graph into the true graph---to complement the F1 results in the main text.
Lower SHD indicates a predicted graph closer to ground truth.

\begin{table}[ht]
\centering
\caption{SHD on sparse SCP benchmarks (mean $\pm$ std over 10 seeds). Best per row in \textbf{bold}. CDNOTS-Gated achieves the lowest SHD at all dimensionalities.}
\label{tab:shd}
\small
\begin{tabular}{l l cccc}
\toprule
$d$ & $T$ & PCMCI & CDNOTS & TCDF & \textbf{CDNOTS-Gated} \\
\midrule
\multirow{4}{*}{20}
& 300  & 104\tiny{$\pm$9} & 19\tiny{$\pm$5} & 32\tiny{$\pm$9} & \textbf{17}\tiny{$\pm$5} \\
& 500  & 108\tiny{$\pm$12} & 17\tiny{$\pm$4} & 26\tiny{$\pm$10} & \textbf{13}\tiny{$\pm$4} \\
& 1000 & 104\tiny{$\pm$16} & 14\tiny{$\pm$4} & 23\tiny{$\pm$10} & \textbf{11}\tiny{$\pm$4} \\
& 2000 & 106\tiny{$\pm$18} & 14\tiny{$\pm$6} & 22\tiny{$\pm$10} & \textbf{9}\tiny{$\pm$4} \\
\midrule
\multirow{4}{*}{30}
& 300  & 226\tiny{$\pm$12} & 31\tiny{$\pm$5} & 49\tiny{$\pm$11} & \textbf{24}\tiny{$\pm$6} \\
& 500  & 227\tiny{$\pm$16} & 30\tiny{$\pm$6} & 40\tiny{$\pm$14} & \textbf{19}\tiny{$\pm$7} \\
& 1000 & 226\tiny{$\pm$9} & 26\tiny{$\pm$8} & 32\tiny{$\pm$15} & \textbf{13}\tiny{$\pm$4} \\
& 2000 & 229\tiny{$\pm$17} & 25\tiny{$\pm$4} & 31\tiny{$\pm$15} & \textbf{11}\tiny{$\pm$4} \\
\midrule
\multirow{4}{*}{50}
& 300  & 642\tiny{$\pm$21} & 66\tiny{$\pm$13} & 85\tiny{$\pm$18} & \textbf{30}\tiny{$\pm$7} \\
& 500  & 638\tiny{$\pm$32} & 69\tiny{$\pm$14} & 62\tiny{$\pm$19} & \textbf{21}\tiny{$\pm$7} \\
& 1000 & 627\tiny{$\pm$18} & 64\tiny{$\pm$12} & 49\tiny{$\pm$15} & \textbf{16}\tiny{$\pm$4} \\
& 2000 & 619\tiny{$\pm$31} & 74\tiny{$\pm$15} & 45\tiny{$\pm$14} & \textbf{18}\tiny{$\pm$6} \\
\midrule
\multirow{4}{*}{70}
& 300  & 1246\tiny{$\pm$34} & 117\tiny{$\pm$21} & 128\tiny{$\pm$19} & \textbf{37}\tiny{$\pm$16} \\
& 500  & 1233\tiny{$\pm$36} & 118\tiny{$\pm$18} & 92\tiny{$\pm$24} & \textbf{26}\tiny{$\pm$10} \\
& 1000 & 1241\tiny{$\pm$25} & 124\tiny{$\pm$25} & 60\tiny{$\pm$21} & \textbf{18}\tiny{$\pm$7} \\
& 2000 & 1210\tiny{$\pm$14} & 139\tiny{$\pm$26} & 52\tiny{$\pm$19} & \textbf{18}\tiny{$\pm$6} \\
\midrule
\multirow{4}{*}{100}
& 300  & 2575\tiny{$\pm$43} & 218\tiny{$\pm$36} & 159\tiny{$\pm$24} & \textbf{52}\tiny{$\pm$25} \\
& 500  & 2530\tiny{$\pm$54} & 243\tiny{$\pm$42} & 120\tiny{$\pm$38} & \textbf{41}\tiny{$\pm$19} \\
& 1000 & 2490\tiny{$\pm$42} & 281\tiny{$\pm$43} & 79\tiny{$\pm$38} & \textbf{27}\tiny{$\pm$12} \\
& 2000 & 2478\tiny{$\pm$47} & 313\tiny{$\pm$54} & 65\tiny{$\pm$34} & \textbf{26}\tiny{$\pm$10} \\
\bottomrule
\end{tabular}
\end{table}

Table~\ref{tab:shd} confirms the F1 findings: CDNOTS-Gated achieves the lowest SHD at all dimensionalities and sample sizes.
The SHD reduction over raw CDNOTS is particularly striking at high $d$ with sufficient samples: at $d=100$/$T=1000$, CDNOTS-Gated achieves SHD\,$=$\,27 vs.\ 281 for CDNOTS---a 90\% reduction in graph edit distance.
PCMCI's SHD grows dramatically with $d$ (reaching ${\sim}2{,}500$ at $d=100$), reflecting its high false positive rate, while TCDF achieves competitive SHD at high $T$ but is less consistent at low sample sizes.

\section{Score-Based Baseline: CUTS+ Comparison}
\label{app:cuts_plus}

Since CUTS+~\citep{cheng2023cuts_plus} outputs a pairwise probability matrix without lag resolution, we report \emph{pair-level} metrics: the lag dimension is collapsed (an edge exists if any lag is active) before computing precision, recall, and F1.
We use their published default settings: threshold~$0.5$, 64~training epochs, GRU hidden size~32, and Gumbel temperature annealing from $\tau=1$ to $\tau=0.1$.
The input window is set to match the true maximum lag ($L=5$).

\begin{table}[ht]
\centering
\caption{Pair-level F1 comparison between \grace{} (CDNOTS-Gated) and CUTS+ on sparse SCP benchmarks (mean $\pm$ std over 10 seeds).
CUTS+ achieves high recall but very low precision, resulting in pair F1 that degrades sharply with dimensionality.
\grace{} results use the same pair-level evaluation (lags collapsed) for fair comparison.}
\label{tab:cuts_plus}
\small
\begin{tabular}{l l cccc}
\toprule
$d$ & \textbf{Method} & $T=300$ & $T=500$ & $T=1000$ & $T=2000$ \\
\midrule
\multirow{2}{*}{20}
& CDNOTS-Gated  & \textbf{.662}\tiny{$\pm$.07} & \textbf{.736}\tiny{$\pm$.07} & \textbf{.813}\tiny{$\pm$.07} & \textbf{.844}\tiny{$\pm$.06} \\
& CUTS+         & .346\tiny{$\pm$.06} & .441\tiny{$\pm$.08} & .510\tiny{$\pm$.05} & .620\tiny{$\pm$.08} \\
\midrule
\multirow{2}{*}{50}
& CDNOTS-Gated  & \textbf{.732}\tiny{$\pm$.07} & \textbf{.814}\tiny{$\pm$.06} & \textbf{.824}\tiny{$\pm$.09} & \textbf{.878}\tiny{$\pm$.04} \\
& CUTS+         & .112\tiny{$\pm$.04} & .139\tiny{$\pm$.05} & .186\tiny{$\pm$.06} & .400\tiny{$\pm$.08} \\
\midrule
\multirow{2}{*}{100}
& CDNOTS-Gated  & \textbf{.791}\tiny{$\pm$.04} & \textbf{.821}\tiny{$\pm$.04} & \textbf{.834}\tiny{$\pm$.05} & \textbf{.859}\tiny{$\pm$.04} \\
& CUTS+         & .035\tiny{$\pm$.02} & .043\tiny{$\pm$.02} & .055\tiny{$\pm$.03} & .117\tiny{$\pm$.07} \\
\bottomrule
\end{tabular}
\end{table}

\begin{table}[ht]
\centering
\caption{Pair-level precision and recall at $T=1000$ (mean over 10 seeds).
CUTS+ opens most gates (high recall) but cannot distinguish true from false edges (low precision), especially at high $d$.}
\label{tab:cuts_plus_prec}
\small
\begin{tabular}{l cc cc cc}
\toprule
& \multicolumn{2}{c}{$d=20$} & \multicolumn{2}{c}{$d=50$} & \multicolumn{2}{c}{$d=100$} \\
\cmidrule(lr){2-3} \cmidrule(lr){4-5} \cmidrule(lr){6-7}
\textbf{Method} & Prec & Rec & Prec & Rec & Prec & Rec \\
\midrule
CDNOTS-Gated & \textbf{.975} & .702 & \textbf{.885} & .797 & \textbf{.922} & .763 \\
CUTS+        & .372 & \textbf{.838} & .105 & \textbf{.903} & .029 & \textbf{.912} \\
\bottomrule
\end{tabular}
\end{table}

Table~\ref{tab:cuts_plus} shows that CUTS+ pair-level F1 degrades sharply with dimensionality: from $0.510$ at $d=20$ to $0.055$ at $d=100$ ($T=1000$).
Table~\ref{tab:cuts_plus_prec} reveals the underlying cause: CUTS+ achieves high recall ($>83\%$) by opening most Gumbel-Softmax gates, but precision collapses from $37\%$ at $d=20$ to $2.9\%$ at $d=100$.
Its Gumbel-Softmax gates produce continuous probabilities that cluster in a narrow range (typically $0.36$--$0.71$), making it difficult to separate true from false edges regardless of threshold choice.
By contrast, \grace{}'s Hard Concrete gates concentrate near exact zero or one (Appendix~\ref{app:gate_hist}), yielding a robust binary decision at the $0.5$ boundary.

\section{Runtime}
\label{app:runtime}

Table~\ref{tab:runtime} and Figure~\ref{fig:runtime} compare wall-clock runtime at $T=1000$.
CDNOTS-Gated adds $2$--$3.5\times$ overhead over raw CDNOTS while achieving substantially higher F1.
The runtime is dominated by the CDNOTS skeleton at high $d$; the gated training phase itself scales linearly in the number of skeleton edges.

\begin{table}[ht]
\centering
\caption{Runtime in seconds at $T=1000$ (mean over 10 seeds). CDNOTS-Gated is practical even at $d=100$.}
\label{tab:runtime}
\small
\begin{tabular}{l rrrrr}
\toprule
\textbf{Method} & $d=20$ & $d=30$ & $d=50$ & $d=70$ & $d=100$ \\
\midrule
PCMCI          & 5    & 10   & 32   & 75   & 187  \\
CDNOTS        & 4    & 7    & 20   & 50   & 138  \\
TCDF           & 4    & 6    & 10   & 12   & 18   \\
\textbf{CDNOTS-Gated}  & 38   & 57   & 112  & 186  & 345  \\
\bottomrule
\end{tabular}
\end{table}

\begin{figure}[ht]
\centering
\includegraphics[width=\textwidth]{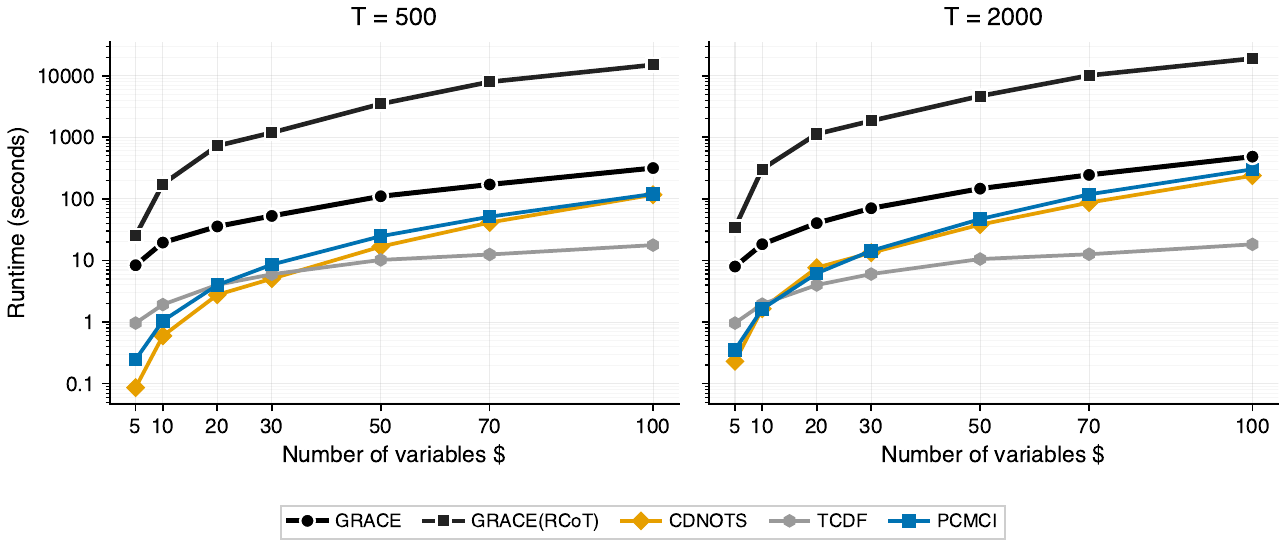}
\caption{Wall-clock runtime vs.\ $d$ (log scale). CDNOTS-Gated (black) adds moderate overhead over CDNOTS while achieving substantially higher F1. CDNOTS(RCoT)-Gated (dark gray) achieves higher F1 but at $50\times$ the cost at $d=100$.}
\label{fig:runtime}
\end{figure}

\section{Ablation Study}
\label{app:ablation}

We ablate three design choices in GRACE on the same SCP benchmarks used in the main experiments ($d \in \{50, 100\}$, $T \in \{500, 1000\}$, 10 seeds each).
The baseline is GRACE with full CDNOTS skeleton and the adaptive $\lambda$ formula.

\paragraph{Skeleton quality.}
Table~\ref{tab:abl_skeleton} evaluates three skeleton configurations:
(i)~\emph{No skeleton}---all $d^2(L+1) - d$ candidate edges are passed to the gated model;
(ii)~\emph{Shallow CI} ($\texttt{max\_degree}=1$)---only pairwise and first-order conditional independence tests;
(iii)~\emph{Full CDNOTS} (default)---conditioning sets up to the maximum graph degree.
Without any skeleton pruning, the gated model faces $\sim$15{,}000 candidate edges at $d=50$ for $\sim$100 true edges; no $\lambda$ setting can recover signal at this ratio (F1\,$<$\,0.02).
Shallow CI improves substantially (F1\,$=$\,0.46--0.71) but still underperforms full CDNOTS by 25--50\% in F1, reflecting the spurious edges that higher-order conditioning would remove.

\begin{table}[ht]
\centering
\caption{Skeleton quality ablation: F1 scores (mean $\pm$ std over 10 seeds).}
\label{tab:abl_skeleton}
\small
\begin{tabular}{l cccc}
\toprule
& \multicolumn{2}{c}{$d=50$} & \multicolumn{2}{c}{$d=100$} \\
\cmidrule(lr){2-3} \cmidrule(lr){4-5}
\textbf{Skeleton} & $T=500$ & $T=1000$ & $T=500$ & $T=1000$ \\
\midrule
No skeleton       & .012\tiny{$\pm$.01} & .013\tiny{$\pm$.01} & .006\tiny{$\pm$.00} & .006\tiny{$\pm$.00} \\
Shallow CI ($k=1$) & .458\tiny{$\pm$.03} & .614\tiny{$\pm$.07} & .697\tiny{$\pm$.07} & .635\tiny{$\pm$.08} \\
\textbf{Full CDNOTS}  & \textbf{.886}\tiny{$\pm$.03} & \textbf{.915}\tiny{$\pm$.02} & \textbf{.878}\tiny{$\pm$.04} & \textbf{.920}\tiny{$\pm$.02} \\
\bottomrule
\end{tabular}

\end{table}

\paragraph{Lambda sensitivity.}
Table~\ref{tab:abl_lambda} compares fixed $\lambda$ values against the adaptive formula (Eq.~\ref{eq:auto_lambda}).
A fixed $\lambda=0.01$ performs comparably at $d=50$, since both yield similar values when the skeleton is sparse.
At $d=100$, the adaptive formula's $1/d$ term lowers $\lambda$ appropriately, maintaining high F1 where a fixed value would overregularize.
$\lambda=0.05$ already degrades at $d=100$ (F1 drops from 0.92 to 0.49), and $\lambda \geq 0.1$ aggressively penalizes gates, suppressing most true edges.

\begin{table}[ht]
\centering
\caption{$\lambda$ sensitivity: F1 scores (mean $\pm$ std over 10 seeds).}
\label{tab:abl_lambda}
\small
\begin{tabular}{l cccc}
\toprule
& \multicolumn{2}{c}{$d=50$} & \multicolumn{2}{c}{$d=100$} \\
\cmidrule(lr){2-3} \cmidrule(lr){4-5}
\textbf{$\lambda$} & $T=500$ & $T=1000$ & $T=500$ & $T=1000$ \\
\midrule
$\lambda=0.01$    & .739\tiny{$\pm$.05} & .846\tiny{$\pm$.04} & .893\tiny{$\pm$.04} & .895\tiny{$\pm$.03} \\
$\lambda=0.05$    & .763\tiny{$\pm$.07} & .765\tiny{$\pm$.07} & .491\tiny{$\pm$.14} & .485\tiny{$\pm$.13} \\
$\lambda=0.1$     & .524\tiny{$\pm$.12} & .529\tiny{$\pm$.11} & .246\tiny{$\pm$.07} & .242\tiny{$\pm$.07} \\
$\lambda=0.5$     & .008\tiny{$\pm$.01} & .015\tiny{$\pm$.02} & .007\tiny{$\pm$.01} & .007\tiny{$\pm$.01} \\
\textbf{Adaptive}   & \textbf{.886}\tiny{$\pm$.03} & \textbf{.915}\tiny{$\pm$.02} & \textbf{.878}\tiny{$\pm$.04} & \textbf{.920}\tiny{$\pm$.02} \\
\bottomrule
\end{tabular}
\end{table}

\paragraph{Lag misspecification robustness.}
\label{app:lag_misspec}
In practice, the true maximum causal lag $L$ is unknown and must be specified by the practitioner.
We evaluate robustness to lag misspecification by generating SCP graphs with true max lag $L = 5$ ($d=50$, $T=1000$, 10 seeds) and running each method with assumed search lags $L_{\text{search}} \in \{2, 3, 5, 7, 10\}$.
Evaluation uses the ground truth padded (or truncated) to $L_{\text{search}}$, so that edges beyond the search horizon count as false negatives for all methods equally.

\begin{table}[ht]
\centering
\caption{F1 vs.\ assumed max lag (true max lag = 5), $d=50$, $T=1000$ (mean $\pm$ std over 10 seeds).}
\label{tab:lag_misspec}
\small
\begin{tabular}{l ccccc}
\toprule
\textbf{Method} & $L=2$ & $L=3$ & $L=5^*$ & $L=7$ & $L=10$ \\
\midrule
GRACE  & .793\tiny{$\pm$.024} & .825\tiny{$\pm$.036} & \textbf{.906}\tiny{$\pm$.022} & .878\tiny{$\pm$.026} & .891\tiny{$\pm$.022} \\
CDNOTS & .626\tiny{$\pm$.030} & .676\tiny{$\pm$.050} & .732\tiny{$\pm$.060} & .714\tiny{$\pm$.056} & .679\tiny{$\pm$.048} \\
\midrule
\multicolumn{6}{l}{\small $^*$True max lag.} \\
\bottomrule
\end{tabular}
\end{table}

Under underspecification ($L_{\text{search}} < 5$), late-lag true edges are outside the search space and count as false negatives for all methods equally.
GRACE degrades from $0.906$ at the correct lag to $0.793$ at $L=2$, while CDNOTS drops further to $0.626$---a gap of $0.167$ vs.\ $0.113$ points respectively, showing that GRACE is more robust even in the underspecified regime.
Under overspecification ($L_{\text{search}} > 5$), the advantage is more striking: GRACE maintains $F_1 = 0.878$--$0.891$ at $L \in \{7, 10\}$ (precision $\geq 0.908$), because the Hard Concrete $L_0$ penalty suppresses the extra false-lag candidates.
CDNOTS has no such suppression mechanism and loses $0.053$ F1 points at $L=10$ relative to $L=5$, as spurious cross-lag edges inflate false positives.

\section{PCMCI-Gated: Refinement Across Skeleton Sources}
\label{app:pcmci_gated}

To demonstrate that \grace{}'s gated refinement is not tied to a specific skeleton algorithm, we replace the CDNOTS skeleton with PCMCI~\citep{runge2019detecting} and apply the same gating pipeline (denoted PCMCI-Gated).
PCMCI provides an ideal stress test: it achieves very high recall ($>90\%$) but poor precision that degrades sharply with dimensionality (from ${\sim}30\%$ at $d=20$ to ${\sim}6\%$ at $d=100$), producing a skeleton with many false positives for gating to prune.

\begin{table}[ht]
\centering
\caption{F1 scores for PCMCI vs.\ PCMCI-Gated on sparse SCP benchmarks (mean $\pm$ std over 10 seeds).
Gated refinement consistently improves PCMCI, with the largest gains at high $d$ and sufficient $T$.}
\label{tab:pcmci_gated_f1}
\small
\begin{tabular}{l l cccc}
\toprule
$d$ & \textbf{Method} & $T=300$ & $T=500$ & $T=1000$ & $T=2000$ \\
\midrule
\multirow{2}{*}{20}
& PCMCI        & .436\tiny{$\pm$.03} & .433\tiny{$\pm$.02} & .455\tiny{$\pm$.03} & .454\tiny{$\pm$.03} \\
& PCMCI-Gated  & .427\tiny{$\pm$.03} & .431\tiny{$\pm$.03} & \textbf{.460}\tiny{$\pm$.04} & \textbf{.615}\tiny{$\pm$.07} \\
\midrule
\multirow{2}{*}{30}
& PCMCI        & .341\tiny{$\pm$.02} & .346\tiny{$\pm$.03} & .354\tiny{$\pm$.03} & .356\tiny{$\pm$.03} \\
& PCMCI-Gated  & .331\tiny{$\pm$.02} & .343\tiny{$\pm$.03} & \textbf{.537}\tiny{$\pm$.05} & \textbf{.733}\tiny{$\pm$.06} \\
\midrule
\multirow{2}{*}{50}
& PCMCI        & .221\tiny{$\pm$.02} & .228\tiny{$\pm$.02} & .236\tiny{$\pm$.02} & .241\tiny{$\pm$.02} \\
& PCMCI-Gated  & .242\tiny{$\pm$.03} & \textbf{.427}\tiny{$\pm$.06} & \textbf{.758}\tiny{$\pm$.07} & \textbf{.752}\tiny{$\pm$.08} \\
\midrule
\multirow{2}{*}{70}
& PCMCI        & .161\tiny{$\pm$.02} & .168\tiny{$\pm$.02} & .171\tiny{$\pm$.02} & .177\tiny{$\pm$.02} \\
& PCMCI-Gated  & \textbf{.345}\tiny{$\pm$.04} & \textbf{.687}\tiny{$\pm$.06} & \textbf{.710}\tiny{$\pm$.07} & \textbf{.696}\tiny{$\pm$.08} \\
\midrule
\multirow{2}{*}{100}
& PCMCI        & .105\tiny{$\pm$.02} & .113\tiny{$\pm$.02} & .117\tiny{$\pm$.02} & .119\tiny{$\pm$.02} \\
& PCMCI-Gated  & \textbf{.605}\tiny{$\pm$.11} & \textbf{.602}\tiny{$\pm$.12} & \textbf{.602}\tiny{$\pm$.13} & \textbf{.527}\tiny{$\pm$.16} \\
\bottomrule
\end{tabular}
\end{table}

\begin{table}[ht]
\centering
\caption{Precision and TPR at $T=1000$ for PCMCI vs.\ PCMCI-Gated (mean over 10 seeds).
Gating dramatically improves precision at the cost of some recall, converting PCMCI's high-recall/low-precision profile into a more balanced one.}
\label{tab:pcmci_gated_prec}
\small
\begin{tabular}{l cc cc cc cc cc}
\toprule
& \multicolumn{2}{c}{$d=20$} & \multicolumn{2}{c}{$d=30$} & \multicolumn{2}{c}{$d=50$} & \multicolumn{2}{c}{$d=70$} & \multicolumn{2}{c}{$d=100$} \\
\cmidrule(lr){2-3} \cmidrule(lr){4-5} \cmidrule(lr){6-7} \cmidrule(lr){8-9} \cmidrule(lr){10-11}
\textbf{Method} & Prec & TPR & Prec & TPR & Prec & TPR & Prec & TPR & Prec & TPR \\
\midrule
PCMCI        & .298 & .973 & .218 & .969 & .135 & .978 & .094 & .984 & .062 & .982 \\
PCMCI-Gated  & \textbf{.334} & .752 & \textbf{.420} & .755 & \textbf{.806} & .719 & \textbf{.743} & .695 & \textbf{.661} & .590 \\
\bottomrule
\end{tabular}
\end{table}

\begin{table}[ht]
\centering
\caption{SHD for PCMCI vs.\ PCMCI-Gated on sparse SCP benchmarks (mean $\pm$ std over 10 seeds).
Lower is better. Gating reduces SHD by up to $17\times$ at high $d$.}
\label{tab:pcmci_gated_shd}
\small
\begin{tabular}{l l rrrr}
\toprule
$d$ & \textbf{Method} & $T=300$ & $T=500$ & $T=1000$ & $T=2000$ \\
\midrule
\multirow{2}{*}{20}
& PCMCI        & 104\tiny{$\pm$8} & 108\tiny{$\pm$12} & 104\tiny{$\pm$15} & 106\tiny{$\pm$17} \\
& PCMCI-Gated  & \textbf{87}\tiny{$\pm$6} & \textbf{88}\tiny{$\pm$8} & \textbf{79}\tiny{$\pm$12} & \textbf{44}\tiny{$\pm$13} \\
\midrule
\multirow{2}{*}{30}
& PCMCI        & 226\tiny{$\pm$12} & 227\tiny{$\pm$15} & 226\tiny{$\pm$9} & 229\tiny{$\pm$17} \\
& PCMCI-Gated  & \textbf{189}\tiny{$\pm$13} & \textbf{182}\tiny{$\pm$11} & \textbf{83}\tiny{$\pm$8} & \textbf{35}\tiny{$\pm$11} \\
\midrule
\multirow{2}{*}{50}
& PCMCI        & 642\tiny{$\pm$20} & 638\tiny{$\pm$31} & 627\tiny{$\pm$17} & 619\tiny{$\pm$30} \\
& PCMCI-Gated  & \textbf{459}\tiny{$\pm$24} & \textbf{192}\tiny{$\pm$22} & \textbf{48}\tiny{$\pm$22} & \textbf{50}\tiny{$\pm$29} \\
\midrule
\multirow{2}{*}{70}
& PCMCI        & 1246\tiny{$\pm$33} & 1233\tiny{$\pm$34} & 1241\tiny{$\pm$24} & 1210\tiny{$\pm$14} \\
& PCMCI-Gated  & \textbf{360}\tiny{$\pm$34} & \textbf{88}\tiny{$\pm$36} & \textbf{78}\tiny{$\pm$33} & \textbf{83}\tiny{$\pm$38} \\
\midrule
\multirow{2}{*}{100}
& PCMCI        & 2575\tiny{$\pm$41} & 2530\tiny{$\pm$52} & 2490\tiny{$\pm$40} & 2478\tiny{$\pm$44} \\
& PCMCI-Gated  & \textbf{145}\tiny{$\pm$70} & \textbf{149}\tiny{$\pm$90} & \textbf{150}\tiny{$\pm$91} & \textbf{234}\tiny{$\pm$190} \\
\bottomrule
\end{tabular}
\end{table}

Tables~\ref{tab:pcmci_gated_f1} and~\ref{tab:pcmci_gated_shd} show that gated refinement consistently improves PCMCI across all dimensionalities.
At $d=50$/$T=1000$, F1 improves from 0.236 to 0.758 ($3.2\times$) and SHD drops from 627 to 48 ($13\times$); at $d=100$/$T=1000$, F1 improves from 0.117 to 0.602 ($5.1\times$) and SHD from 2490 to 150 ($17\times$).
Table~\ref{tab:pcmci_gated_prec} reveals the mechanism: gating converts PCMCI's extreme recall/low-precision profile (e.g., 97\% TPR / 6\% precision at $d=100$) into a more balanced one (59\% TPR / 66\% precision).

Comparing with CDNOTS-Gated (Table~\ref{tab:f1_main}), PCMCI-Gated achieves lower F1 at all configurations---e.g., 0.602 vs.\ 0.920 at $d=100$/$T=1000$---because PCMCI's skeleton has far more false positives than CDNOTS's.
This confirms that \grace{} delivers consistent and substantial gains regardless of the skeleton source---up to $5\times$ F1 improvement---while the final performance ceiling is set by the skeleton's recall.

\section{Score-Based Skeleton: DYNOTEARS}
\label{app:dynotears_gated}

The preceding appendix demonstrated that \grace{} improves CI-based skeletons (PCMCI).
Here we test whether the same gating mechanism can refine a \emph{score-based} skeleton produced by DYNOTEARS~\citep{pamfil2020dynotears}, which fits an $L_1$-penalized vector autoregression (LASSO VAR) per target variable using cross-validated penalty selection.
LASSO VAR produces continuous coefficient magnitudes, and the central challenge is converting them to a binary graph: setting the threshold too low retains many false positives, while setting it too high discards true edges.
This is precisely the thresholding problem that \grace{}'s Hard Concrete gates solve automatically.

\paragraph{Oracle threshold analysis.}
To quantify the thresholding difficulty, we sweep all possible coefficient thresholds on each seed and record the value that maximizes F1 given the ground truth (the ``oracle'' threshold).
Across all configurations ($d \in \{20, 50, 100\}$, $T \in \{500, 1000\}$, 10 seeds each), the oracle threshold clusters in the range 0.04--0.09 but varies up to $5\times$ across seeds (e.g., 0.038--0.188 at $d=20$).
A fixed threshold of $t=0.05$ recovers ${\sim}97\%$ of the oracle F1, making it a reasonable lower bound---but selecting it required access to the ground truth, which is unavailable in practice.

\paragraph{Results.}
Table~\ref{tab:dynotears_gated} compares four configurations: DYNOTEARS with the default threshold ($t=0$, all nonzero coefficients), DYNOTEARS with the post-hoc oracle-informed threshold ($t=0.05$), DYNOTEARS-Gated (GRACE applied to the $t=0.05$ skeleton), and CDNOTS-Gated (the default GRACE pipeline).

\begin{table}[ht]
\centering
\caption{F1 scores for DYNOTEARS variants and CDNOTS-Gated on sparse SCP benchmarks (mean $\pm$ std over 10 seeds).
$\dagger$\,Post-hoc threshold selected via oracle analysis; unavailable in practice.}
\label{tab:dynotears_gated}
\small
\begin{tabular}{l cccccc}
\toprule
& \multicolumn{2}{c}{$d=20$} & \multicolumn{2}{c}{$d=50$} & \multicolumn{2}{c}{$d=100$} \\
\cmidrule(lr){2-3} \cmidrule(lr){4-5} \cmidrule(lr){6-7}
\textbf{Method} & $T=500$ & $T=1000$ & $T=500$ & $T=1000$ & $T=500$ & $T=1000$ \\
\midrule
DYNOTEARS ($t=0$)              & .238\tiny{$\pm$.04} & .222\tiny{$\pm$.03} & .168\tiny{$\pm$.03} & .150\tiny{$\pm$.03} & .149\tiny{$\pm$.03} & .129\tiny{$\pm$.03} \\
DYNOTEARS ($t=0.05$)$^\dagger$ & .861\tiny{$\pm$.05} & .892\tiny{$\pm$.05} & .843\tiny{$\pm$.05} & .911\tiny{$\pm$.04} & .822\tiny{$\pm$.03} & .902\tiny{$\pm$.03} \\
DYNO-Gated ($t=0.05$)$^\dagger$ & .871\tiny{$\pm$.04} & .907\tiny{$\pm$.04} & .866\tiny{$\pm$.05} & \textbf{.924}\tiny{$\pm$.03} & .824\tiny{$\pm$.04} & .908\tiny{$\pm$.03} \\
\textbf{CDNOTS-Gated}           & .833\tiny{$\pm$.04} & .862\tiny{$\pm$.04} & .825\tiny{$\pm$.08} & .897\tiny{$\pm$.04} & \textbf{.888}\tiny{$\pm$.04} & \textbf{.907}\tiny{$\pm$.03} \\
\bottomrule
\end{tabular}
\end{table}

Several patterns emerge.
First, DYNOTEARS with the default threshold ($t=0$) achieves F1 of only 0.13--0.24, confirming that raw LASSO coefficients without thresholding are unusable as a causal graph---the vast majority of nonzero coefficients are false positives.
Second, the oracle threshold $t=0.05^\dagger$ dramatically improves DYNOTEARS to F1 0.82--0.91, demonstrating that a good threshold \emph{exists} but requires ground truth to find.
Third, DYNOTEARS-Gated ($t=0.05^\dagger$) matches or slightly exceeds CDNOTS-Gated at $d \leq 50$, showing that \grace{} can effectively refine score-based skeletons when the skeleton quality is sufficient.
Fourth, CDNOTS-Gated retains its advantage at $d=100$/$T=500$ (0.888 vs.\ 0.824), where the LASSO skeleton becomes sparser and misses more true edges.

The key takeaway is that \grace{}'s Hard Concrete gates provide the same benefit to score-based skeletons as to CI-based ones: automatic, threshold-free conversion of continuous edge scores into a binary causal graph.
However, the DYNOTEARS-Gated results rely on a post-hoc threshold ($\dagger$) to pre-filter the skeleton before gating.
By contrast, CDNOTS-Gated requires no such threshold---the CI-based skeleton is already binary---making it the preferred default in practice.
This experiment confirms that \grace{} is compatible with diverse skeleton sources: it delivers consistent improvements over skeletons from both CI-based methods (CDNOTS, PCMCI) and score-based methods (DYNOTEARS), with the gating mechanism providing the same threshold-free binary decision benefit across all skeleton types.

\section{Lorenz-96 Benchmark: Multiplicative Interactions}
\label{app:lorenz96}

The Lorenz-96 model~\citep{lorenz1996predictability} is a chaotic dynamical system defined by the ODE
$\dot{x}_i = (x_{i+1} - x_{i-2}) \cdot x_{i-1} - x_i + F$,
where indices wrap cyclically.
Each variable $x_i$ has four causal parents at lag~1: $x_{i-2}$, $x_{i-1}$, $x_i$ (self), and $x_{i+1}$.
Critically, the mechanism is \emph{multiplicative}---$x_{i-1}$ multiplies $(x_{i+1} - x_{i-2})$---creating interaction effects that violate GRACE's additive parent assumption.
This benchmark directly tests whether GRACE remains useful when its structural assumptions are misspecified.

We use the same implementation parameters as CUTS+~\citep{cheng2023cuts_plus}: forcing $F=10$ (chaotic regime), integration time step $\Delta t=0.1$ via \texttt{scipy.integrate.odeint}, additive Gaussian observation noise $\sigma=0.1$, and burn-in of 1000~steps.
The observed time series are thus noisy discretizations of the underlying chaotic ODE, making the causal discovery task more realistic and challenging.

\begin{table}[ht]
\centering
\caption{F1 on Lorenz-96 with noisy observations ($\sigma=0.1$), mean $\pm$ std over 10 seeds.
GRACE substantially improves over the CDNOTS skeleton at higher $N$, demonstrating that gated refinement remains effective even when the true causal mechanism is multiplicative.}
\label{tab:lorenz96_f1}
\small
\begin{tabular}{l l ccc}
\toprule
$N$ & \textbf{Method} & $T=500$ & $T=1000$ & $T=2000$ \\
\midrule
\multirow{2}{*}{50}
& CDNOTS  & .757\tiny{$\pm$.01} & .748\tiny{$\pm$.01} & .760\tiny{$\pm$.01} \\
& GRACE   & \textbf{.764}\tiny{$\pm$.01} & \textbf{.828}\tiny{$\pm$.01} & \textbf{.874}\tiny{$\pm$.01} \\
\midrule
\multirow{2}{*}{100}
& CDNOTS  & .724\tiny{$\pm$.01} & .717\tiny{$\pm$.01} & .723\tiny{$\pm$.01} \\
& GRACE   & \textbf{.840}\tiny{$\pm$.01} & \textbf{.875}\tiny{$\pm$.00} & \textbf{.894}\tiny{$\pm$.01} \\
\bottomrule
\end{tabular}
\end{table}

\begin{table}[ht]
\centering
\caption{Precision and recall on Lorenz-96 at $T=1000$ (mean over 10 seeds).
GRACE preserves the skeleton's recall while dramatically improving precision, achieving perfect precision at $N=100$.}
\label{tab:lorenz96_prec}
\small
\begin{tabular}{l cc cc}
\toprule
& \multicolumn{2}{c}{$N=50$} & \multicolumn{2}{c}{$N=100$} \\
\cmidrule(lr){2-3} \cmidrule(lr){4-5}
\textbf{Method} & Prec & Rec & Prec & Rec \\
\midrule
CDNOTS  & .717 & .782 & .665 & .778 \\
GRACE   & \textbf{.879} & .782 & \textbf{1.000} & .778 \\
\bottomrule
\end{tabular}
\end{table}

Tables~\ref{tab:lorenz96_f1} and~\ref{tab:lorenz96_prec} show that GRACE substantially improves over CDNOTS on Lorenz-96 at higher dimensionality.
At $N=100$/$T=1000$, GRACE achieves F1\,$=$\,0.875 with perfect precision (1.000) compared to CDNOTS F1\,$=$\,0.717 with precision 0.665---a pattern consistent with the SCP results in Section~\ref{sec:experiments}.
Both methods share the same recall (${\sim}0.78$), confirming that GRACE prunes false positives from the skeleton without removing true edges.
The improvement grows with both $N$ and $T$: at $N=50$/$T=500$ the gap is small (F1 0.764 vs.\ 0.757), but at $N=100$/$T=2000$ it is substantial (0.894 vs.\ 0.723).

These results demonstrate that GRACE's additive MLP decoders can effectively distinguish true from spurious edges even when the underlying causal mechanism is multiplicative.
While the model cannot perfectly represent the true functional form $x_{i-1} \cdot (x_{i+1} - x_{i-2})$, it learns sufficient signal from true parent variables to keep their gates open, while spurious edges provide no predictive value and are pruned.
The observation noise ($\sigma=0.1$) further validates robustness: the method operates on noisy discretizations of the chaotic ODE, not clean trajectories.

\section{Diverse Graph Topologies}
\label{app:topology}

The main experiments (Section~\ref{sec:experiments}) use sparse random SCPs where edges are sampled uniformly at random.
A natural concern is whether \grace{} generalizes to fundamentally different graph structures.
We test on three standard graph families generated via \texttt{NetworkX}~\citep{hagberg2008exploring}:
\begin{itemize}
\item \textbf{Scale-free} (Barab\'asi--Albert, $m=2$): hub-and-spoke topology with power-law degree distribution, characteristic of gene regulatory networks and metabolic pathways where a few master regulators influence many downstream targets~\citep{barabasi1999emergence}.
\item \textbf{Erd\H{o}s--R\'enyi} ($p=0.08$, directed): uniformly random edges, producing a denser graph (${\sim}200$ cross-edges at $d=50$) with no preferential structure.
\item \textbf{Small-world} (Watts--Strogatz, $k=4$, $p=0.3$): clustered ring lattice with random rewiring, combining high local clustering with short path lengths, as observed in neural connectomes and social networks~\citep{watts1998collective}.
\end{itemize}
For undirected generators (BA, WS), edges are directed from lower to higher index to establish a causal ordering.
Lags are drawn uniformly from $\{1, \ldots, 5\}$, and causal mechanisms use the same coefficient ranges and mixed linear/nonlinear function pool as the SCP benchmarks.
Autoregressive terms and spectral radius checking ensure stationarity.

We evaluate at $d=50$, $T=1000$ (8--10 seeds per configuration) using pair-level AUROC, which is threshold-free and enables a fair comparison between \grace{}'s continuous gate values and CUTS+'s Gumbel-Softmax probabilities.

\begin{table}[ht]
\centering
\caption{Pair-level AUROC on diverse graph topologies ($d=50$, $T=1000$, mean $\pm$ std).
\grace{} leads on Erd\H{o}s--R\'enyi and small-world graphs; CUTS+ edges ahead on scale-free.}
\label{tab:topology}
\small
\begin{tabular}{l cc}
\toprule
\textbf{Topology} & \textbf{\grace{}} & \textbf{CUTS+} \\
\midrule
Scale-free      & .902\tiny{$\pm$.05} & \textbf{.932}\tiny{$\pm$.02} \\
Erd\H{o}s--R\'enyi & \textbf{.839}\tiny{$\pm$.02} & .803\tiny{$\pm$.02} \\
Small-world     & \textbf{.976}\tiny{$\pm$.01} & .966\tiny{$\pm$.02} \\
\bottomrule
\end{tabular}
\end{table}

Table~\ref{tab:topology} shows that \grace{} generalizes well across all three topologies.
On small-world and Erd\H{o}s--R\'enyi graphs \grace{} leads in AUROC (0.976 and 0.839 vs.\ 0.966 and 0.803).
Scale-free is the one topology where CUTS+ achieves slightly higher AUROC (0.932 vs.\ 0.902), possibly because the hub-and-spoke structure suits CUTS+'s GRU message-passing architecture.
The method maintains its advantage across fundamentally different graph structures---including the denser Erd\H{o}s--R\'enyi graphs (edge density ${\sim}4\times$ higher than SCP)---demonstrating that gated refinement generalizes beyond the sparse random SCPs used in the main experiments.

\section{Elbe River Network: Graph Comparison}
\label{app:elbe_compare}

Figure~\ref{fig:elbe_compare} visualizes the causal graphs recovered on the Elbe River's main branch (Section~\ref{sec:causalrivers}).
The CDNOTS skeleton (center) recovers all 11 true edges but adds 106 spurious connections, making the graph uninterpretable.
\grace{}-Bootstrap (right) prunes nearly all false positives while retaining the chain structure.

\begin{figure}[ht]
\centering
\includegraphics[width=\textwidth]{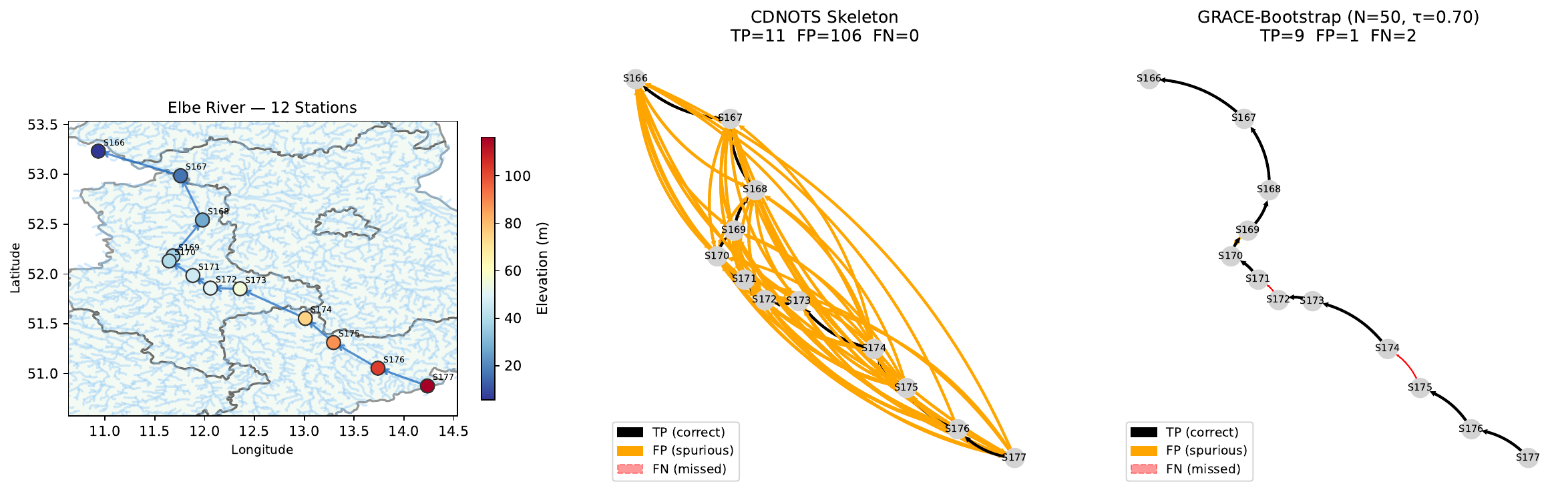}
\caption{Elbe River's main branch ($d=12$)~\citep{stein2025causalrivers}. \textbf{Left:} geographic map of 12 gauging stations colored by elevation, with arrows indicating true flow direction (upstream $\to$ downstream). \textbf{Center:} CDNOTS skeleton---all 11 true edges recovered (black) but buried under 106 false positives (orange). \textbf{Right:} \grace{}-Bootstrap ($N=50$, $\tau=0.70$)---9 true edges retained with only 1 false positive, recovering the river's chain structure.}
\label{fig:elbe_compare}
\end{figure}

\end{document}